\documentclass{article}
\PassOptionsToPackage{numbers,compress}{natbib}
\usepackage[preprint]{neurips_2026}
\usepackage[utf8]{inputenc}
\usepackage[T1]{fontenc}    
\usepackage[colorlinks=true, citecolor=RoyalBlue, linkcolor=RoyalBlue, urlcolor=RoyalBlue]{hyperref}

\usepackage{url}            
\usepackage{booktabs}       
\usepackage{makecell}       
\usepackage{amsfonts}       
\usepackage{nicefrac}       
\usepackage{microtype}     
\usepackage[table,dvipsnames]{xcolor} 
\usepackage{graphicx}       
\usepackage{tikz}
\usetikzlibrary{calc}
\usepackage{amsthm}
\usepackage{amsmath}
\usepackage{enumitem}
\usepackage{tabularx}
\usepackage{placeins}
\newtheorem{proposition}{Proposition}
\newcolumntype{P}{>{\hsize=0.65\hsize\centering\arraybackslash}X}
\newcolumntype{R}{>{\hsize=1.05\hsize\hspace{0.85em}\centering\arraybackslash}X}
\newcolumntype{Q}{>{\hsize=1.075\hsize\centering\arraybackslash}X}
\definecolor{Accent}{HTML}{146B8F}
\colorlet{Accent15}{Accent!15!white}
\definecolor{GainRowGray}{HTML}{E8E8E8}
\definecolor{GainGreen}{HTML}{0D5C2E}
\definecolor{GainRed}{HTML}{8B1A1A}
\newcommand{\sectionrule}{\specialrule{0.6pt}{0pt}{0pt}}
\newcommand{\doublemidrule}{
  \specialrule{\heavyrulewidth}{3pt}{1.5pt}
  \specialrule{\heavyrulewidth}{0pt}{3pt}
}
\newcommand{\taberr}[1]{\mbox{\normalfont\scalebox{0.72}{$\pm#1$}}}
\newcommand{\tcgamart}[1]{\cellcolor{Accent15}#1}
\newcommand{\relgaindash}[1]{
  \specialrule{0.5pt}{2.5pt}{2.5pt}
}
  
\setlist[itemize]{leftmargin=1.35em,itemsep=0.2em,topsep=0.3em}
\title{Martingale-Consistent Self-Supervised Learning}

\author{
Moritz Gögl$^{1}$ \quad Hanwen Xing$^{1}$ \quad Christopher Yau$^{1,2}$ \\
$^1$University of Oxford\\
$^2$Health Data Research UK\\
\texttt{moritz.gogl@keble.ox.ac.uk}
}

\begin{document}  
\maketitle
\begin{abstract}
Self-supervised learning (SSL) is often deployed under changing information, such as shorter histories, missing features, or partially observed images. In these settings, predictions from coarse and refined views should be coherent: before refinement, the coarse-view prediction should match the average prediction expected after refinement. Martingales formalize this coherence principle, but standard SSL objectives do not enforce it. Unlike invariance objectives that pull views together, martingale consistency constrains only the expected refined prediction, allowing predictions to update as information is revealed while preventing systematic drift. We introduce a martingale-consistent SSL framework that closes this gap, with practical prediction- and latent-space variants and an unbiased two-sample Monte Carlo estimator based on stochastic refinement. We evaluate the approach on synthetic and real time-series, tabular, and image benchmarks under partial-observation regimes, in both semi-self-supervised and fully label-free settings. Across these experiments, our framework improves robustness and calibration under partial observation, yielding more stable representations as information is revealed.
\end{abstract}

\section{Introduction}
\label{sec:intro}
Self-supervised learning (SSL) has become a standard approach for learning reusable representations from unlabeled data \cite{Uelwer2025_SSLSurvey,Chen2020_SimCLR,Grill2020_BYOL,He2022_MAE}. However, while SSL has been studied through probabilistic lenses such as mutual information maximisation, latent variable modeling, and masked conditional prediction, existing SSL objectives do not typically treat learned predictors as a coherent family of conditional expectations indexed by nested information sets. In practice, a learned predictor is often queried under multiple information sets corresponding to shorter or longer histories, coarser or finer contexts, delayed or missing features, or different acquisition budgets. A single model is then expected to behave coherently as information is revealed, withheld, or refined.

\paragraph{Illustrative Example:}
Suppose a model predicts whether a patient will develop sepsis within the next 24 hours $(Y \in \{0,1\})$. The input $x$ consists of clinical measurements including temperature, heart rate, blood pressure, white blood cell count, and lactate level. Different information sets correspond to different stages of data availability during care. For example, there could be three nested observation sets: (i) $F_1$: vital signs only, (ii) $F_2$: vital signs + routine laboratory tests, and (iii) $F_M$: all available measurements, satisfying $F_1 \subseteq F_2 \subseteq F_M$. A predictor $g$ may estimate sepsis risk as $g(x_{F_1}) = 0.30$ from vital signs alone and revise this estimate after laboratory results become available. Before the laboratory tests are observed, we would rationally expect that the refined prediction conditioned on the currently available information equals the coarse prediction:
\begin{equation*}
\mathbb{E}[g(x_{F_2}) \mid x_{F_1}] = g(x_{F_1}).
\end{equation*}
This is the \emph{martingale property}.\nopagebreak

Intuitively, the refined prediction may increase or decrease substantially for a particular patient once laboratory results arrive. However, before observing those tests, the model should not systematically expect its future predictions to drift upward or downward. If the conditional expectation instead equals $0.45\!\neq \!0.30$, then the model is already implicitly underestimating risk at the coarse-information \nolinebreak stage.

In this paper, we refer to the enforcement of this martingale property in SSL as \emph{martingale consistency} \cite{Williams1991_ProbabilityWithMartingales}. Standard SSL objectives do not enforce this consistency \cite{Chen2020_SimCLR,Grill2020_BYOL,He2022_MAE} and can therefore learn predictors that perform well for individual views while remaining globally inconsistent across nested information sets. We develop a probabilistic view of SSL under changing information in which predictors form a family of conditional expectations indexed by information structure, and we enforce coherence across information refinement during training. 

\paragraph{Contributions.} Our contributions are threefold: (1) we identify martingale consistency as a defining–yet missing–structural property of SSL viewed as a conditional expectation family: standard SSL applies a single predictor across many conditioning sets without enforcing the consistency relations required of a valid conditional expectation operator, and we introduce this property as a training criterion; (2) we provide a practical training recipe – prediction-space and latent-space martingale penalties with an unbiased two-sample Monte Carlo estimator – that closes this gap without modality-specific redesign; and (3) across temporal, tabular, and image benchmarks in both semi-self-supervised and fully self-supervised settings, we demonstrate that closing this gap improves downstream robustness, probabilistic calibration, and representation stability under partial \nolinebreak observation.

\section{Related Work}
\paragraph{SSL across objectives and modalities.}
SSL has developed along several main objective families. In vision, contrastive and bootstrap methods such as SimCLR~\cite{Chen2020_SimCLR} and BYOL~\cite{Grill2020_BYOL} learn representations by aligning augmented views, while masked reconstruction methods such as MAE learn from incomplete inputs by predicting masked content~\cite{He2022_MAE}. In time series, TS2Vec~\cite{Yue2022_TS2Vec} uses hierarchical contrastive learning, while Ti-MAE~\cite{Li2023_TiMAE} and SimMTM~\cite{Dong2023_SimMTM} adopt masked modeling objectives. In tabular learning, VIME~\cite{Yoon2020_VIME}, SCARF~\cite{Bahri2022_SCARF}, and SubTab~\cite{Ucar2021_SubTab} adapt self-supervision through corruption-based pretext tasks, contrastive learning, and feature-subset views, respectively.

\paragraph{Consistency across views.}
Many SSL methods rely on agreement across multiple views, but typically through invariance or representation alignment. SimCLR~\cite{Chen2020_SimCLR} and BYOL~\cite{Grill2020_BYOL} encourage similar representations across augmentations, and VICReg~\cite{Bardes2022_VICReg} makes this explicit through invariance, variance, and covariance regularization. Masked modeling methods likewise couple views through reconstruction~\cite{He2022_MAE,Li2023_TiMAE,Dong2023_SimMTM}. These constraints ask that views resemble each other; they do not constrain how predictions update as information is refined. Martingale consistency instead enforces a directional identity - the coarse prediction equals the expected refined prediction - which is the defining relation of a conditional expectation family. It is complementary to existing SSL objectives rather than a replacement: as we will demonstrate in our experiments, adding it on top of SimCLR and BYOL improves downstream robustness over the vanilla objectives, indicating it captures structure these objectives leave unconstrained.

\section{Conditional Coherence in Self-Supervised Learning}
\label{sec:methods}
\subsection{Problem setup}
Let $(X_t)_{t=1}^T$ be a stochastic process with joint distribution $p(x_{1:T})$. We aim to learn representations of partial histories that are useful for predicting future variables under arbitrary conditioning structures. Let $\mathcal{F}$ denote the information available through a particular observed view (e.g., a masked or subsampled history). For a target random variable $Y = h(X_{t+1:t+k})$, the Bayes-optimal predictor given $\mathcal{F}$ is the conditional expectation
$m_{\mathcal{F}} := E[Y | \mathcal{F}]$. 

A fundamental property of conditional expectations is the \emph{martingale property}: for any $\mathcal{F}_1 \subseteq \mathcal{F}_2$,
\begin{equation}
\label{eq:martingale}
E[m_{\mathcal{F}_2} | \mathcal{F}_1] = m_{\mathcal{F}_1}.
\end{equation}
This property is illustrated in Figure~\ref{fig:martingale}. It states that averaging the refined-view prediction over the information not yet observed in $\mathcal{F}_1$ recovers the prediction based only on $\mathcal{F}_1$. Intuitively, before observing any additional information, our expectation of how the prediction will change after refinement should be unchanged. A predictor that adheres to this will be called \emph{martingale consistent}.

\begin{figure}[t]
\centering
\vspace{-0.2cm}
\includegraphics[width=1.0\textwidth]{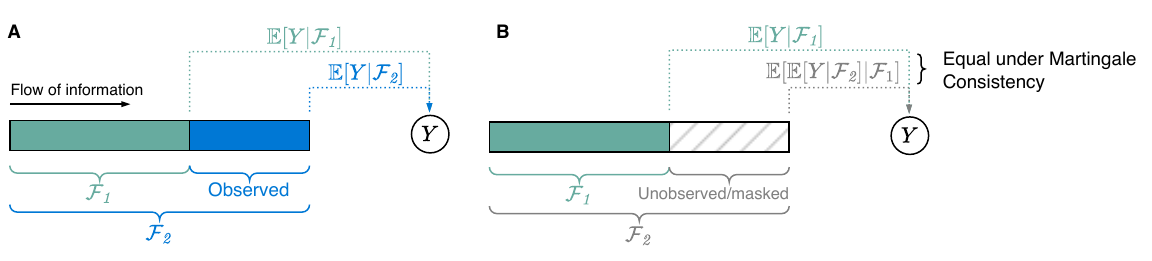}
\vspace{-0.5cm}
\caption{Illustration of the Martingale Property. (A)~The prediction for target $Y$ is updated from available information. (B) The martingale property says the conditional expectation of $Y$ given no further information should agree with that based only on what is observed. Consistency is achieved when this is true across all nested pairs $(\mathcal{F}_1, \mathcal{F}_2)$.}
\vspace{-0.15cm}
\label{fig:martingale}
\end{figure}

\subsection{Application to self-supervision}
Standard SSL methods can be interpreted as learning predictors under a particular conditioning structure \cite{Chen2020_SimCLR,Grill2020_BYOL,He2022_MAE}. Let $Y$ denote a target random variable associated with $X$, either externally provided or derived from $X$ itself. Each training objective approximates a conditional expectation \nolinebreak of \nolinebreak the  form
\begin{equation}
g_\phi(x_{\mathcal{F}}) \approx E[Y \mid \mathcal{F}],
\end{equation}
where $\mathcal{F}$ is induced by a masking scheme, subsampling rule, or causal history on $X$.

After training, however, the same predictor $g_\phi$ is typically applied across many partial-observation patterns $x_{\mathcal{F}}$. Implicitly, this treats $g_\phi$ as a family of conditional predictors indexed by $\mathcal{F}$. Yet standard SSL objectives train each conditioning independently and do not enforce coherence across nested information sets. Even when training exposes the model to a distribution over masks or partial-information patterns, the learning problem still decomposes across observed views, and predictions made under different views are not explicitly coupled. Our proposal is to impose the martingale identity directly as a training objective.

From a representation-learning perspective, this acts as a structural inductive bias: by penalizing incoherence across masks, subsampling schemes, or temporal scales, the model is encouraged to learn representations that remain stable under information marginalization, improving robustness under changing observation patterns.

\subsection{Martingale-consistent objective}
We now formalize the martingale-consistency objective underlying our approach. Let $P(\mathcal{F})$ be the distribution over information patterns induced by masks, subsampling, or history-selection rules. Each draw $\mathcal{F}\sim P$ corresponds to a partial observation $x_{\mathcal{F}}$. We can then learn a predictor
\begin{equation}
g_\phi(x_{\mathcal{F}}) \approx E[Y | \mathcal{F}],
\end{equation}
parameterized by $\phi$, typically implemented as a neural network head on top of a shared encoder.

\paragraph{Predictive term.} We include a standard self-supervised predictive loss
\begin{equation}
\mathcal{L}_{\text{pred}} = E \left[ \ell(g_\phi(X_{\mathcal{F}}), Y) \right],
\end{equation}
where $\ell$ is mean squared error or negative log-likelihood. In purely self-supervised settings, this often specializes to the case $Y=t(X)$ for an appropriate target map $t$.

\paragraph{Martingale consistency term.} To enforce coherence across conditionings, we sample nested information sets $\mathcal{F}_1\! \subseteq\!\mathcal{F}_2$ (e.g., coarse and fine masks). We penalize deviations from the martingale \nolinebreak identity
\vspace{-0.2cm}
\begin{equation}
\mathcal{L}_{\text{mart}} = E_{\mathcal{F}_1 \subseteq \mathcal{F}_2} \left[ \left\| g_\phi(x_{\mathcal{F}_1}) - E[g_\phi(x_{\mathcal{F}_2}) | x_{\mathcal{F}_1}] \right\|^2 \right].
\end{equation}
The full objective is $\mathcal{L} = \mathcal{L}_{\text{pred}} + \lambda \mathcal{L}_{\text{mart}}.$ Taken alone, however, $\mathcal{L}_{\mathrm{mart}}$ does not identify a useful predictor. For example, any mask-independent constant predictor $g_\phi(x_{\mathcal{F}})\equiv c$ satisfies
\begin{equation}
g_\phi(x_{\mathcal{F}_1}) - E[g_\phi(x_{\mathcal{F}_2}) \mid x_{\mathcal{F}_1}]
= c - c = 0,
\end{equation}
and hence incurs zero martingale penalty despite being uninformative. More generally, replacing $g_\phi$ by $g_\phi + c$ for any constant vector $c$ leaves $\mathcal{L}_{\mathrm{mart}}$ unchanged. Consequently, $\mathcal{L}_{\mathrm{mart}}$ must be paired with a predictive objective that anchors the encoder to meaningful targets. Appendix~\ref{app:theory:linear} characterizes this inductive bias in closed form: increasing $\lambda$ progressively recovers the optimal coarse-view predictor.

\paragraph{Difference from standard consistency regularization and invariance learning.} Conventional symmetric view-alignment or consistency objectives encourage predictions or representations from different views to be directly similar, e.g. $g(x_{F_1}) \approx g(x_{F_2})$, thereby suppressing view-dependent variation. In contrast, martingale consistency constrains only the conditional expectation of refined predictions $g(x_{F_1}) \approx \mathbb{E}[g(x_{F_2}) \mid x_{F_1}]$, allowing individual refined predictions to vary substantially as additional information is revealed while preventing systematic drift in expectation. The objective therefore preserves information-sensitive updates rather than enforcing invariance across views.

\vspace{-0.10cm}
\subsection{Approximating conditional expectations}
\vspace{-0.10cm}
To make the martingale objective practical, we must approximate the conditional expectation $E[g_\phi(x_{\mathcal{F}_2})\!\mid\!x_{\mathcal{F}_1}]$. A natural strategy is to draw independent samples from a refinement mechanism that is conditionally consistent with $x_{\mathcal{F}_1}$ and use these samples to build a Monte Carlo estimator. However, replacing $E[g_\phi(x_{\mathcal{F}_2}) \mid x_{\mathcal{F}_1}]$ by a finite-sample average inside $\mathcal{L}_{\mathrm{mart}}$ introduces bias through Monte Carlo variance. We propose an \emph{unbiased two-sample Monte Carlo estimator} which draws two conditionally independent refinements $x_{\mathcal{F}_2}^{(a)}$ and $x_{\mathcal{F}_2}^{(b)}$ given $x_{\mathcal{F}_1}$, define
$u := g_\phi(x_{\mathcal{F}_1})$, $v_a := g_\phi(x_{\mathcal{F}_2}^{(a)})$, and $v_b := g_\phi(x_{\mathcal{F}_2}^{(b)})$, and set
\begin{equation}
\widehat{\mathcal{L}}_{\mathrm{mart}}^{2\mathrm{MC}}=
\left( u - v_a \right)^\top
\left( u - v_b \right).
\end{equation}
Since the two samples are independent and each has expectation $E[g_\phi(x_{\mathcal{F}_2}) \mid x_{\mathcal{F}_1}]$, taking the conditional expectation given $x_{\mathcal{F}_1}$ yields
\begin{equation}
E\left[ \widehat{\mathcal{L}}_{\mathrm{mart}}^{2\mathrm{MC}} \mid x_{\mathcal{F}_1} \right]=\left\| g_\phi(x_{\mathcal{F}_1}) - E[g_\phi(x_{\mathcal{F}_2}) \mid x_{\mathcal{F}_1}] \right\|^2.
\label{eq:2mc}
\end{equation}
Taking the further expectation over $x_{\mathcal{F}_1}$ recovers $\mathcal{L}_{\mathrm{mart}}$, so $\widehat{\mathcal{L}}_{\mathrm{mart}}^{2\mathrm{MC}}$ is an unbiased estimator of $\mathcal{L}_{\mathrm{mart}}$, not a different objective (for proof, see Appendix \ref{app:proof_unbiasedness}). Although it has higher variance than a large-sample Monte Carlo average, it remains unbiased even with a single draw per branch. 

\subsection{Latent-space martingale consistency}
\vspace{-0.1cm}
In high-dimensional settings such as images, video, or long sequences, enforcing martingale consistency directly in observation space may be impractical. Self-supervised models instead operate through an encoder that maps partial observations to lower-dimensional latent representations, so it is natural to impose consistency at that level. Formally, let $f_\theta$ be an encoder mapping a partial observation $x_{\mathcal{F}}$ to a latent representation $z_{\mathcal{F}} = f_\theta(x_{\mathcal{F}})$. This representation retains a coarsened version of the information in the input view. We consider predictors of the form
$g_\phi(z_{\mathcal{F}}) \approx E[Y | z_{\mathcal{F}}],$
and enforce martingale consistency across nested conditionings $\mathcal{F}_1 \subseteq \mathcal{F}_2$ at the level of latent representations: $E[g_\phi(z_{\mathcal{F}_2}) | z_{\mathcal{F}_1}] = g_\phi(z_{\mathcal{F}_1})$.
We posit this as a latent-space surrogate for Eq.~(\ref{eq:martingale}); it coincides exactly with the population martingale identity when $z_{\mathcal{F}_1}$ generates $\mathcal{F}_1$.
This enforces coherence with respect to the information retained by the encoder rather than the full observation space. In the reported experiments, the latent-space constraint is implemented through stochastic completion and, in EMA variants, refined-path target networks. Latent-space martingale consistency is therefore a scalable surrogate for the ideal observation-space constraint. Its guarantees are relative to the encoder: consistency is enforced only for the information preserved in the latent space, not for all information in the raw input. This trade-off is standard in modern self-supervised and masked autoencoding architectures \cite{He2022_MAE, Assran2023_iJEPA, Grill2020_BYOL}.

\section{Instantiating the Martingale-Consistent SSL Framework}
\label{sec:instantiation}
\vspace{-0.10cm}

We now translate the framework of Section~\ref{sec:methods} into a concrete training recipe. The framework requires only a nested pair $\mathcal{F}_1 \subseteq \mathcal{F}_2$, and $\mathcal{F}_2$ may be any refinement of $\mathcal{F}_1$. However, without loss of generality, in the experiments that follow, we specialize to the most natural choice under missingness: $\mathcal{F}_2 = x$ is the complete input - the maximally informative training view - and $\mathcal{F}_1 = x_{\text{obs}}$ is the masked partial view:
\begin{equation}
x_{\mathcal{F}_1} := x_{\mathrm{obs}} := x \odot M,
\qquad
x_{\mathcal{F}_2} := x.
\end{equation}
Figure~\ref{fig:exp_pipeline} depicts the resulting architecture for time-series data under right-censoring. A shared encoder $f_\theta$ and prediction head $g_\phi$, applied to the complete input, produce the base predictive loss $\mathcal{L}_{\mathrm{pred}}$, which is a classification loss in the semi-self-supervised setting and, e.g., a reconstruction loss in the label-free setting.

\begin{figure}[t!]
  \centering
\vspace{-0.2cm}
\includegraphics[width=\linewidth]{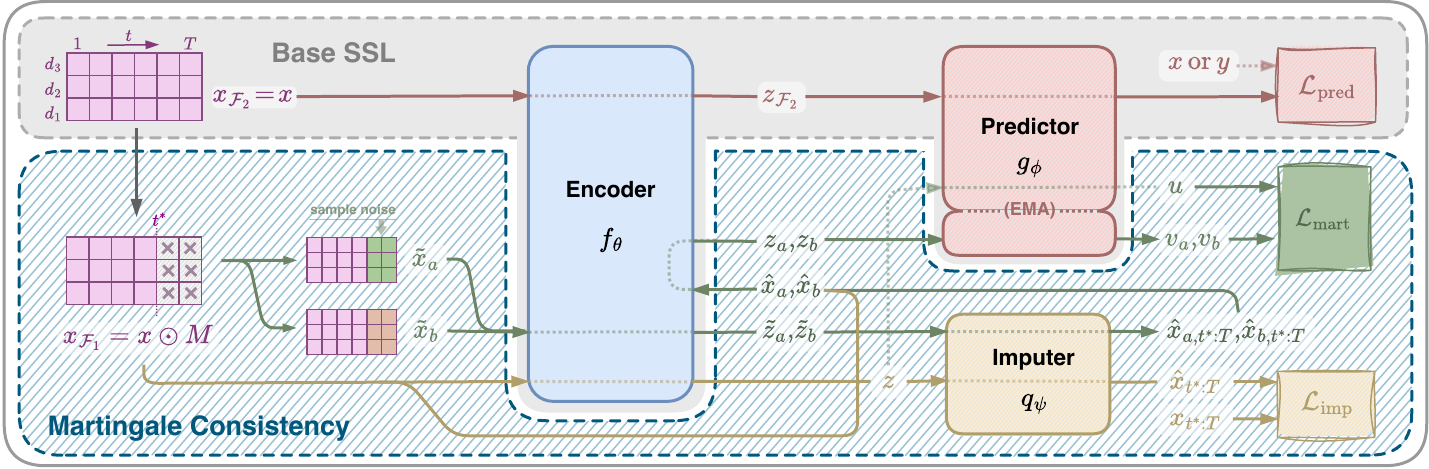}
\vspace{-0.4cm}
  \caption{Base model architecture for martingale-consistent SSL visualized for the right-censored time-series setting. The predictive branch uses the fully observed input $x_{\mathcal{F}_2}\!=\!x$, while the martingale branch starts from $x_{\mathcal{F}_1}\!=\!x_{\mathrm{obs}}\!=\!x \odot M$, forms two imputed refinements $\hat{x}_a,\hat{x}_b$ via $q_\psi$, and instantiates the martingale estimator $\widehat{\mathcal{L}}_{\mathrm{mart}}^{2\mathrm{MC}}$ in prediction or latent space; $\mathcal{L}_{\mathrm{imp}}$ is applied only on removed entries.}
  \vspace{-0.1cm}
  \label{fig:exp_pipeline}
\end{figure}

\paragraph{Imputer.} Applying Eq.~\eqref{eq:2mc} requires two draws from $p(x_{\mathcal{F}_2} \mid x_{\mathcal{F}_1})$ – the conditional distribution of the complete input given the partial view. In general this distribution is unknown, so we introduce an imputer $q_\psi$: a learned stochastic completion network that approximates $p(x_{\mathcal{F}_2} \mid x_{\mathcal{F}_1})$ by filling in the unobserved coordinates given $x_{\mathcal{F}_1}$. It is trained with an auxiliary loss $\mathcal{L}_{\mathrm{imp}}$ applied only to the removed entries, keeping its training signal disentangled from the predictive objective. Given $x_{\mathcal{F}_1}$, two independent draws from $q_\psi$ yield two stochastic completions $\hat{x}_a$ and $\hat{x}_b$: each merges the observed entries of $x$ with independently imputed draws for the missing coordinates, providing the approximately conditionally independent refinements that Eq.~\eqref{eq:2mc} requires.

\paragraph{Prediction-space and latent-space variants.} Writing $z\!=\!f_\theta(x_{\mathcal{F}_1})$ and $z_a\!=\!f_\theta(\hat{x}_a)$, $z_b\!=\!f_\theta(\hat{x}_b)$, prediction-space variants substitute $u = g_\phi(z)$ for the coarse prediction and $v_a = g_\phi(z_a)$, $v_b = g_\phi(z_b)$ for the two refined predictions, giving the estimator $\widehat{\mathcal{L}}_{\mathrm{mart,pred}}^{2\mathrm{MC}}$. Latent-space variants instead use these representations directly:
\begin{equation}
\widehat{\mathcal{L}}_{\mathrm{mart,lat}}^{2\mathrm{MC}}=\left(z-z_a\right)^\top\left(z-z_b\right),
\end{equation}
where $z$ is the coarse latent representation and $z_a,z_b$ are the two \nolinebreak refined representations. This concretizes the latent-space surrogate of Section~\ref{sec:methods} via imputer-based refinement.

\paragraph{EMA target networks.} Both variants admit an exponential moving average (EMA) extension in which the refined targets ($v_a, v_b$ or $z_a, z_b$) are computed with a slowly-updating copy of the encoder or prediction head. Its parameters $\theta_{\mathrm{EMA}}$ track the main network via $\theta_{\mathrm{EMA}} \leftarrow \tau\,\theta_{\mathrm{EMA}} + (1-\tau)\,\theta$ after each step, preventing the moving-target instability that arises when the same network determines both the coarse prediction and the refined targets it is optimized to match. 

\vspace{0.1cm}
The combined training objective is given by
\begin{equation}
\mathcal{L}_{\mathrm{total}}=\mathcal{L}_{\mathrm{pred}}
+\lambda_{\mathrm{imp}}\,\mathcal{L}_{\mathrm{imp}}
+\lambda_{\mathrm{mart}}\,\mathcal{L}_{\mathrm{mart}},
\end{equation}
where $\mathcal{L}_{\mathrm{mart}}$ is implemented by $\widehat{\mathcal{L}}_{\mathrm{mart,pred}}^{2\mathrm{MC}}$ or $\widehat{\mathcal{L}}_{\mathrm{mart,lat}}^{2\mathrm{MC}}$ depending on the variant. $\mathcal{L}_{\mathrm{imp}}$ is part of the imputation-based refinement mechanism rather than the abstract martingale objective: it trains $q_\psi$ to produce the stochastic completions that instantiate the two-sample estimator. Full implementation details are provided in Appendix~\ref{app:impl}.
\vspace{-0.05cm}
\section{Experiments}
\label{sec:experiments}
\vspace{-0.12cm}
\subsection{Overview}
\vspace{-0.12cm}
We evaluate our martingale-consistent SSL framework, testing whether enforcing conditional coherence leads to more reliable representations as the amount of observed information varies. Downstream robustness is measured by probe accuracy across varying levels of observation completeness.

The experiments are organized along two high-level axes. First, we distinguish between a \emph{semi-self-supervised} setting, in which class labels define the predictive term, and a \emph{fully self-supervised} setting, in which the main training signal is based on reconstruction, i.e., label-free. Second, we evaluate our approach across various modalities: We begin with \textit{time-series} data where right-censoring gives the clearest concrete picture of a martingale: observing a longer prefix corresponds directly to a refinement of information. We then move to more complex temporal missingness patterns and feature revelation in \textit{tabular} and \textit{image} data, which instantiate the same nested-information relation $\mathcal{F}_1 \subseteq \mathcal{F}_2$ without requiring a strictly sequential interpretation. This illustrates that the framework is not tied to temporal data and can extend to other modalities with evolving or partial information. Finally, Appendix~\ref{app:tcga} reports a complementary application in a TCGA multi-omics survival study.

\subsection{Baselines and Model Ablations}
Relative to the base SSL architecture in Figure~\ref{fig:exp_pipeline}, our simplest baseline keeps only the encoder and predictor and optimizes the underlying SSL objective with $\lambda_{\mathrm{mart}}=\lambda_{\mathrm{imp}}=0$ (\textit{Base}). We also consider a non-martingale imputation baseline (\textit{Base + imputation}), which includes the imputation module and its auxiliary loss ($\lambda_{\mathrm{imp}}>0$) but sets $\lambda_{\mathrm{mart}}=0$. The martingale extensions come in the four variants arising from Section~\ref{sec:instantiation}: \textit{Base + mart.} applies the martingale penalty in prediction space; \textit{Base + mart. (EMA)} uses the same prediction-space penalty but computes the refined targets with an EMA copy of the prediction head; \textit{Base + mart. (latent)} applies the martingale penalty in latent space; \textit{Base + mart. (latent, EMA)} combines latent-space regularization with a full-model EMA target (encoder and prediction head). In the fully self-supervised static setting, we also include custom tabular adaptations of \textit{SimCLR} \cite{Chen2020_SimCLR} and \textit{BYOL} \cite{Grill2020_BYOL} – self-supervised objectives originally developed for images (SimCLR uses a contrastive objective, BYOL a bootstrap objective; see Appendix~\ref{app:impl:hyperparams} for tabular adaptation details) – and extend each with a latent-space martingale consistency term enforcing cross-conditioning coherence on top of the primary objective.

\subsection{Benchmark families and datasets}
The empirical study combines three main benchmark families:
\vspace{-0.1cm}
\paragraph{Time-series benchmarks.} We start with simulated time-series data (\emph{T-SIM-RC}) under \emph{right-censored} missingness, where only a prefix of each sequence is observed. We then study a structured temporal simulation (\emph{T-SIM}), where entries can be missing at different time points according to a richer temporal pattern, and finally move to real time-series datasets: \emph{UCI HAR (HAR)} \cite{Anguita2013_UCIHAR}, \emph{Traffic (TRAF)} \cite{Cuturi2011_PEMSSF}, \emph{Stock (STK)}, \emph{Character Trajectories (CT)} \cite{Williams2006_CharacterTrajectories}, and \emph{Spoken Arabic Digits (SAD)} \cite{BeddaHammami2008_SpokenArabicDigit}.
\vspace{-0.1cm}
\paragraph{Static tabular benchmarks.} We also study simulated static data (\emph{S-SIM}) and real tabular datasets: \emph{Adult (AD)} \cite{BeckerKohavi1996_Adult}, \emph{Bank Marketing (BM)} \cite{MoroCortezRita2014_BankMarketing}, \emph{Credit-g (CG)} \cite{Hofmann1994_CreditG}, and \emph{Phoneme (PH)} \cite{HastieBujaTibshirani1995_Phoneme}. Here, partial observations act over features rather than over time.
\vspace{-0.1cm}
\paragraph{Image benchmarks.} To test whether the same conditional-coherence idea remains in high-dimensional visual settings, we also evaluate CIFAR-10 \cite{Krizhevsky2009_CIFAR10} and STL-10 \cite{Gordon2011_STL-10} under center-biased patch masking, where the coarse view reveals only part of the image.
\vspace{0.1cm}

Apart from the right-censored time-series and image benchmarks, training masks are drawn from a prior inferred on a held-out prior-fitting split, rather than sampled directly from the target evaluation process. This keeps SSL training separate from the final held-out test evaluation while still exposing models to realistic partial-observation patterns. Within each training batch, the custom training masks randomize the observed-information rate over $[0.05, 1.0]$, ensuring models are trained across both mild and severe partial-observation regimes. Appendix~\ref{app:data} summarizes the datasets, while Appendices~\ref{app:data:partial} and~\ref{app:impl:masking} detail the partial-observation processes and their training-time  approximations. 

\subsection{Results}
\vspace{-0.10cm}
Our main scalar summary is mean downstream probe accuracy over completeness levels (information fraction observed) $c \in \{0.05, 0.2, 0.4, 0.6, 0.8\}$.
\paragraph{Time-series benchmarks.}
\vspace{-0.10cm} Table~\ref{tab:results_temporal_class} shows a consistent benefit from martingale regularization on the time-series benchmarks. In the semi-self-supervised setting, a martingale variant attains the best mean downstream accuracy on all 7 datasets, with the largest relative gains over Base on \emph{T-SIM-RC} ($+30.9\%$), \emph{CT} ($+155.6\%$), and \emph{SAD} ($+37.8\%$). The non-martingale imputation baseline is sometimes mildly beneficial – on \emph{STK} it ties the best martingale variant. In the fully self-supervised time-series setting, martingale regularization attains the best result on all seven datasets, with especially clear gains on \emph{T-SIM-RC} ($+14.6\%$), \emph{TRAF} ($+12.6\%$), and \emph{SAD} ($+17.3\%$). Overall, the time-series experiments support the claim that enforcing cross-information coherence improves robustness under reduced history.
\begin{table*}[t!]
  \vspace{-0.2cm}
  \centering
  \scriptsize
  \caption{Downstream accuracies on the time-series benchmarks. Each dataset contributes mean downstream accuracy over completeness levels $0.05$ to $0.8$, shown as mean $\pm$ SEM across 5 runs.}
  \begingroup
  \setlength{\tabcolsep}{4pt}
  \begin{tabular}{lccccccc}
    \toprule
    & \textbf{T-SIM-RC} & \textbf{T-SIM} & \textbf{HAR} & \textbf{TRAF} & \textbf{STK} & \textbf{CT} & \textbf{SAD} \\
    \midrule
    \multicolumn{8}{c}{\textit{Semi-self-supervised}} \\
    \midrule
    Base & 0.320{\taberr{6.5\mathrm{e}^{-3}}} & 0.266{\taberr{1.0\mathrm{e}^{-2}}} & 0.867{\taberr{4.4\mathrm{e}^{-3}}} & 0.859{\taberr{5.5\mathrm{e}^{-3}}} & 0.564{\taberr{1.2\mathrm{e}^{-2}}} & 0.153{\taberr{1.8\mathrm{e}^{-2}}} & 0.291{\taberr{8.4\mathrm{e}^{-3}}} \\
    Base (with imputation) & 0.333{\taberr{5.5\mathrm{e}^{-3}}} & 0.267{\taberr{1.0\mathrm{e}^{-2}}} & 0.867{\taberr{4.2\mathrm{e}^{-3}}} & 0.851{\taberr{5.7\mathrm{e}^{-3}}} & \textbf{0.582{\taberr{4.7\mathrm{e}^{-3}}}} & 0.156{\taberr{1.3\mathrm{e}^{-2}}} & 0.291{\taberr{9.4\mathrm{e}^{-3}}} \\
    \rowcolor{Accent15} Base + mart. & \textbf{0.419{\taberr{2.5\mathrm{e}^{-3}}}} & \textbf{0.315{\taberr{6.5\mathrm{e}^{-3}}}} & 0.883{\taberr{2.2\mathrm{e}^{-3}}} & 0.824{\taberr{3.0\mathrm{e}^{-2}}} & 0.581{\taberr{5.2\mathrm{e}^{-3}}} & 0.344{\taberr{1.1\mathrm{e}^{-2}}} & 0.387{\taberr{8.6\mathrm{e}^{-3}}} \\
    \rowcolor{Accent15} Base + mart. (EMA) & 0.375{\taberr{1.0\mathrm{e}^{-2}}} & 0.313{\taberr{3.5\mathrm{e}^{-3}}} & \textbf{0.888{\taberr{2.7\mathrm{e}^{-3}}}} & \textbf{0.905{\taberr{9.7\mathrm{e}^{-3}}}} & 0.570{\taberr{1.2\mathrm{e}^{-2}}} & \textbf{0.391{\taberr{9.9\mathrm{e}^{-3}}}} & \textbf{0.401{\taberr{2.8\mathrm{e}^{-3}}}} \\
    \rowcolor{Accent15} Base + mart. (latent) & 0.414{\taberr{2.1\mathrm{e}^{-3}}} & 0.312{\taberr{4.1\mathrm{e}^{-4}}} & 0.887{\taberr{3.5\mathrm{e}^{-3}}} & 0.836{\taberr{2.7\mathrm{e}^{-2}}} & 0.581{\taberr{7.9\mathrm{e}^{-3}}} & 0.268{\taberr{7.8\mathrm{e}^{-3}}} & 0.359{\taberr{4.8\mathrm{e}^{-3}}} \\
    \rowcolor{Accent15} Base + mart. (latent, EMA) & 0.399{\taberr{7.5\mathrm{e}^{-3}}} & 0.294{\taberr{5.3\mathrm{e}^{-3}}} & 0.886{\taberr{1.6\mathrm{e}^{-3}}} & 0.849{\taberr{2.5\mathrm{e}^{-2}}} & \textbf{0.582{\taberr{4.0\mathrm{e}^{-3}}}} & 0.296{\taberr{8.2\mathrm{e}^{-3}}} & 0.314{\taberr{5.9\mathrm{e}^{-3}}} \\
    \relgaindash{1-8}
    \rowcolor{GainRowGray}\textit{Rel.\ gain (vs Base)} & \textcolor{GainGreen}{+30.9\%} & \textcolor{GainGreen}{+18.4\%} & \textcolor{GainGreen}{+2.4\%} & \textcolor{GainGreen}{+5.4\%} & \textcolor{GainGreen}{+3.2\%} & \textcolor{GainGreen}{+155.6\%} & \textcolor{GainGreen}{+37.8\%} \\
    \doublemidrule
    \multicolumn{8}{c}{\textit{Fully self-supervised}} \\
    \midrule
    Base & 0.356{\taberr{8.4\mathrm{e}^{-3}}} & 0.288{\taberr{3.0\mathrm{e}^{-3}}} & 0.751{\taberr{7.6\mathrm{e}^{-3}}} & 0.555{\taberr{1.2\mathrm{e}^{-2}}} & 0.636{\taberr{7.9\mathrm{e}^{-3}}} & 0.137{\taberr{1.2\mathrm{e}^{-2}}} & 0.295{\taberr{5.5\mathrm{e}^{-3}}} \\
    Base (with imputation) & 0.362{\taberr{8.1\mathrm{e}^{-3}}} & 0.291{\taberr{2.8\mathrm{e}^{-3}}} & 0.752{\taberr{7.3\mathrm{e}^{-3}}} & 0.552{\taberr{1.1\mathrm{e}^{-2}}} & 0.637{\taberr{6.3\mathrm{e}^{-3}}} & 0.135{\taberr{1.2\mathrm{e}^{-2}}} & 0.283{\taberr{6.1\mathrm{e}^{-3}}} \\
    \rowcolor{Accent15} Base + mart. & 0.406{\taberr{3.2\mathrm{e}^{-3}}} & 0.307{\taberr{1.2\mathrm{e}^{-3}}} & 0.753{\taberr{7.3\mathrm{e}^{-3}}} & \textbf{0.625{\taberr{6.1\mathrm{e}^{-3}}}} & 0.637{\taberr{8.5\mathrm{e}^{-3}}} & 0.133{\taberr{1.1\mathrm{e}^{-2}}} & 0.283{\taberr{6.6\mathrm{e}^{-3}}} \\
    \rowcolor{Accent15} Base + mart. (EMA) & 0.404{\taberr{2.5\mathrm{e}^{-4}}} & 0.308{\taberr{6.2\mathrm{e}^{-4}}} & 0.752{\taberr{7.8\mathrm{e}^{-3}}} & 0.604{\taberr{2.1\mathrm{e}^{-2}}} & 0.648{\taberr{0}} & \textbf{0.175{\taberr{9.6\mathrm{e}^{-3}}}} & 0.249{\taberr{1.2\mathrm{e}^{-2}}} \\
    \rowcolor{Accent15} Base + mart. (latent) & 0.406{\taberr{9.4\mathrm{e}^{-4}}} & 0.307{\taberr{5.8\mathrm{e}^{-4}}} & 0.746{\taberr{7.3\mathrm{e}^{-3}}} & 0.584{\taberr{9.7\mathrm{e}^{-3}}} & 0.636{\taberr{6.9\mathrm{e}^{-3}}} & 0.137{\taberr{1.3\mathrm{e}^{-2}}} & 0.305{\taberr{3.4\mathrm{e}^{-3}}} \\
    \rowcolor{Accent15} Base + mart. (latent, EMA) & \textbf{0.408{\taberr{2.3\mathrm{e}^{-3}}}} & \textbf{0.310{\taberr{2.5\mathrm{e}^{-5}}}} & \textbf{0.754{\taberr{7.0\mathrm{e}^{-3}}}} & 0.556{\taberr{8.9\mathrm{e}^{-3}}} & \textbf{0.650{\taberr{1.4\mathrm{e}^{-2}}}} & 0.139{\taberr{1.3\mathrm{e}^{-2}}} & \textbf{0.346{\taberr{6.6\mathrm{e}^{-3}}}} \\
    \relgaindash{1-8}
    \rowcolor{GainRowGray}\textit{Rel.\ gain (vs Base)} & \textcolor{GainGreen}{+14.6\%} & \textcolor{GainGreen}{+7.6\%} & \textcolor{GainGreen}{+0.4\%} & \textcolor{GainGreen}{+12.6\%} & \textcolor{GainGreen}{+2.2\%} & \textcolor{GainGreen}{+27.7\%} & \textcolor{GainGreen}{+17.3\%} \\
    \bottomrule
  \end{tabular}
  \endgroup
  \label{tab:results_temporal_class}
  \label{tab:results_temporal_label_free}
  \vspace{-0.1cm}
\end{table*}

\vspace{-0.15cm}
\paragraph{Static tabular benchmarks.}

\begin{table*}[b!]
  \vspace{-0.4cm}
  \centering
  \scriptsize
  \caption{Downstream accuracies on the static tabular benchmarks. Each dataset contributes mean downstream accuracy over completeness levels $0.05$ to $0.8$, shown as mean $\pm$ SEM across 5 runs.}
  \begingroup
  \setlength{\tabcolsep}{10.9pt}%
  \begin{tabular}{lccccc}
    \toprule
    & \textbf{S-SIM} & \textbf{AD} & \textbf{BM} & \textbf{CG} & \textbf{PH} \\
    \midrule
    \multicolumn{6}{c}{\textit{Semi-self-supervised}} \\
    \midrule
    Base & 0.328{\taberr{5.4\mathrm{e}^{-3}}} & 0.786{\taberr{1.5\mathrm{e}^{-2}}} & 0.814{\taberr{3.4\mathrm{e}^{-2}}} & 0.701{\taberr{8.4\mathrm{e}^{-3}}} & 0.735{\taberr{6.2\mathrm{e}^{-3}}} \\
    Base (with imputation) & 0.368{\taberr{9.9\mathrm{e}^{-3}}} & 0.784{\taberr{1.7\mathrm{e}^{-2}}} & 0.823{\taberr{3.4\mathrm{e}^{-2}}} & 0.690{\taberr{9.9\mathrm{e}^{-3}}} & 0.746{\taberr{2.1\mathrm{e}^{-3}}} \\
    \rowcolor{Accent15} Base + mart. & 0.423{\taberr{4.4\mathrm{e}^{-3}}} & \textbf{0.810{\taberr{7.2\mathrm{e}^{-4}}}} & \textbf{0.895{\taberr{2.5\mathrm{e}^{-4}}}} & 0.716{\taberr{8.1\mathrm{e}^{-3}}} & 0.752{\taberr{1.9\mathrm{e}^{-3}}} \\
    \rowcolor{Accent15} Base + mart. (EMA) & 0.432{\taberr{3.3\mathrm{e}^{-3}}} & 0.809{\taberr{6.6\mathrm{e}^{-4}}} & 0.894{\taberr{1.8\mathrm{e}^{-4}}} & 0.692{\taberr{8.3\mathrm{e}^{-3}}} & 0.753{\taberr{2.2\mathrm{e}^{-3}}} \\
    \rowcolor{Accent15} Base + mart. (latent) & 0.431{\taberr{3.3\mathrm{e}^{-3}}} & 0.807{\taberr{4.6\mathrm{e}^{-4}}} & 0.891{\taberr{4.3\mathrm{e}^{-4}}} & 0.720{\taberr{6.9\mathrm{e}^{-3}}} & \textbf{0.755{\taberr{2.1\mathrm{e}^{-3}}}} \\
    \rowcolor{Accent15} Base + mart. (latent, EMA) & \textbf{0.440{\taberr{2.0\mathrm{e}^{-3}}}} & 0.790{\taberr{1.7\mathrm{e}^{-2}}} & 0.892{\taberr{1.7\mathrm{e}^{-3}}} & \textbf{0.724{\taberr{5.3\mathrm{e}^{-3}}}} & 0.749{\taberr{4.1\mathrm{e}^{-3}}} \\
    \relgaindash{1-6}
    \rowcolor{GainRowGray}\textit{Rel.\ gain (vs Base)} & \textcolor{GainGreen}{+34.1\%} & \textcolor{GainGreen}{+3.1\%} & \textcolor{GainGreen}{+10.0\%} & \textcolor{GainGreen}{+3.3\%} & \textcolor{GainGreen}{+2.7\%} \\
    \doublemidrule
    \multicolumn{6}{c}{\textit{Fully self-supervised}} \\
    \midrule
    Base & 0.411{\taberr{1.9\mathrm{e}^{-3}}} & 0.787{\taberr{8.2\mathrm{e}^{-4}}} & 0.888{\taberr{6.0\mathrm{e}^{-4}}} & 0.724{\taberr{8.8\mathrm{e}^{-4}}} & 0.720{\taberr{1.5\mathrm{e}^{-3}}} \\
    Base (with imputation) & 0.410{\taberr{3.0\mathrm{e}^{-3}}} & 0.789{\taberr{7.3\mathrm{e}^{-4}}} & 0.887{\taberr{9.8\mathrm{e}^{-4}}} & 0.721{\taberr{1.9\mathrm{e}^{-3}}} & 0.733{\taberr{2.6\mathrm{e}^{-3}}} \\
    \rowcolor{Accent15} Base + mart. & 0.410{\taberr{3.5\mathrm{e}^{-3}}} & \textbf{0.799{\taberr{7.3\mathrm{e}^{-4}}}} & \textbf{0.889{\taberr{1.7\mathrm{e}^{-4}}}} & \textbf{0.725{\taberr{1.1\mathrm{e}^{-3}}}} & \textbf{0.734{\taberr{2.5\mathrm{e}^{-3}}}} \\
    \rowcolor{Accent15} Base + mart. (EMA) & 0.413{\taberr{1.1\mathrm{e}^{-3}}} & 0.796{\taberr{9.9\mathrm{e}^{-4}}} & 0.888{\taberr{3.3\mathrm{e}^{-4}}} & \textbf{0.725{\taberr{1.1\mathrm{e}^{-3}}}} & 0.733{\taberr{2.2\mathrm{e}^{-3}}} \\
    \rowcolor{Accent15} Base + mart. (latent) & 0.421{\taberr{4.4\mathrm{e}^{-3}}} & 0.797{\taberr{9.9\mathrm{e}^{-4}}} & \textbf{0.889{\taberr{1.6\mathrm{e}^{-4}}}} & 0.722{\taberr{3.0\mathrm{e}^{-3}}} & 0.725{\taberr{2.5\mathrm{e}^{-3}}} \\
    \rowcolor{Accent15} Base + mart. (latent, EMA) & \textbf{0.432{\taberr{1.4\mathrm{e}^{-3}}}} & 0.797{\taberr{2.5\mathrm{e}^{-4}}} & 0.886{\taberr{3.1\mathrm{e}^{-4}}} & 0.723{\taberr{0}} & 0.720{\taberr{9.7\mathrm{e}^{-4}}} \\
    \relgaindash{1-6}
    \rowcolor{GainRowGray}\textit{Rel.\ gain (vs Base)} & \textcolor{GainGreen}{+5.1\%} & \textcolor{GainGreen}{+1.5\%} & \textcolor{GainGreen}{+0.1\%} & \textcolor{GainGreen}{+0.1\%} & \textcolor{GainGreen}{+1.9\%} \\
    \sectionrule
    SimCLR & 0.390{\taberr{8.5\mathrm{e}^{-3}}} & 0.785{\taberr{1.8\mathrm{e}^{-3}}} & 0.884{\taberr{3.0\mathrm{e}^{-4}}} & 0.622{\taberr{4.9\mathrm{e}^{-2}}} & 0.716{\taberr{3.9\mathrm{e}^{-3}}} \\
    SimCLR (imputation only) & \textbf{0.414{\taberr{2.4\mathrm{e}^{-3}}}} & 0.794{\taberr{1.9\mathrm{e}^{-3}}} & \textbf{0.885{\taberr{5.4\mathrm{e}^{-4}}}} & 0.691{\taberr{2.4\mathrm{e}^{-3}}} & \textbf{0.724{\taberr{3.4\mathrm{e}^{-3}}}} \\
    \rowcolor{Accent15} SimCLR + mart. (latent) & \textbf{0.414{\taberr{3.2\mathrm{e}^{-3}}}} & \textbf{0.800{\taberr{6.0\mathrm{e}^{-4}}}} & 0.884{\taberr{5.5\mathrm{e}^{-4}}} & \textbf{0.720{\taberr{2.7\mathrm{e}^{-4}}}} & 0.713{\taberr{2.6\mathrm{e}^{-3}}} \\
    \relgaindash{1-6}
    \rowcolor{GainRowGray}\textit{Rel.\ gain (vs vanilla SimCLR)} & \textcolor{GainGreen}{+6.2\%} & \textcolor{GainGreen}{+1.9\%} & \textcolor{GainGreen}{+0.0\%} & \textcolor{GainGreen}{+15.8\%} & \textcolor{GainRed}{-0.4\%} \\
    \sectionrule
    BYOL & 0.296{\taberr{1.7\mathrm{e}^{-2}}} & 0.660{\taberr{4.4\mathrm{e}^{-2}}} & 0.723{\taberr{5.6\mathrm{e}^{-2}}} & 0.661{\taberr{3.5\mathrm{e}^{-2}}} & 0.665{\taberr{2.2\mathrm{e}^{-2}}} \\
    BYOL (imputation only) & 0.393{\taberr{9.2\mathrm{e}^{-3}}} & 0.795{\taberr{2.4\mathrm{e}^{-3}}} & 0.865{\taberr{5.7\mathrm{e}^{-3}}} & 0.694{\taberr{1.9\mathrm{e}^{-3}}} & \textbf{0.746{\taberr{2.4\mathrm{e}^{-3}}}} \\
    \rowcolor{Accent15} BYOL + mart. (latent) & \textbf{0.411{\taberr{3.0\mathrm{e}^{-3}}}} & \textbf{0.798{\taberr{8.6\mathrm{e}^{-4}}}} & \textbf{0.883{\taberr{5.0\mathrm{e}^{-5}}}} & \textbf{0.719{\taberr{4.6\mathrm{e}^{-3}}}} & 0.734{\taberr{1.8\mathrm{e}^{-3}}} \\
    \relgaindash{1-6}
    \rowcolor{GainRowGray}\textit{Rel.\ gain (vs vanilla BYOL)} & \textcolor{GainGreen}{+38.9\%} & \textcolor{GainGreen}{+20.9\%} & \textcolor{GainGreen}{+22.1\%} & \textcolor{GainGreen}{+8.8\%} & \textcolor{GainGreen}{+10.4\%} \\
    \bottomrule
  \end{tabular}
  \endgroup
\label{tab:results_static_class}
   \vspace{-0.15cm}
\end{table*}

Table~\ref{tab:results_static_class} summarizes five static tabular benchmarks under feature-wise missingness. In the semi-self-supervised block, a martingale variant attains the best mean accuracy on every dataset, with the largest relative improvements over \emph{Base} on \emph{S-SIM} ($+34.1\%$) and \emph{BM} ($+10.0\%$). In the fully self-supervised block, \emph{Base + martingale} variants improves over plain \emph{Base} on all five datasets; the relative gains are smaller than in the semi-self-supervised setting but remain uniformly positive. For tabular SimCLR and BYOL, martingale regularization improves downstream accuracy relative to the vanilla objectives indicating that the additional constraints captures information that SimCLR) and BYOL do not. Imputation-only training remains competitive in predictive performance on some datasets but could be incurring larger martingale consistency violations (Figure~\ref{fig:probe_curves_selected}).

\vspace{-0.15cm}
\paragraph{Robustness curves and martingale violations.}
Figure~\ref{fig:probe_curves_selected} complements the tables with robustness curves for the simulated benchmarks together with prediction-space and latent-space martingale violations across completeness. For each completeness level $c$, we define
\begingroup\small
\begin{equation*}
\widehat{\mathcal{V}}_{\mathrm{pred}}(c)\!:=\!\frac{1}{N d_g}\sum_{i=1}^N\left\|u_i(c)\!-\!\frac{1}{K}\sum_{k=1}^K v_i^{(k)}(c)\right\|_2^2\!\!,
\qquad \hspace{-0.2cm}
\widehat{\mathcal{V}}_{\mathrm{lat}}(c)\!:=\!\frac{1}{N d_z}\sum_{i=1}^N\bigl(z_i(c)-z_i^{(a)}(c)\bigr)^{\!\top}\!\bigl(z_i(c)\!-\!z_i^{(b)}(c)\bigr).
\end{equation*}
\endgroup
where $u_i(c)$ is the coarse-view prediction, $v_i^{(k)}(c)$ are refined predictions, and $z_i^{(a)}(c),z_i^{(b)}(c)$ are two independently refined latent representations. On simulated data, the prediction-space violation can be compared directly to the true conditional expectation induced by the known data-generating process, approximated by Monte Carlo sampling; no analogous population quantity exists in latent space, where the violation is estimated by passing two independent stochastic completions via the imputer through the encoder and measuring representation mismatch. The curves show that the gains are consistently concentrated in the low-completeness regime, precisely where prediction coherence under information refinement is most critical, while the martingale objective often reduces violations in both spaces. This is plausible because the encoder and prediction head are coupled: regularizing outputs can stabilize the representation, while regularizing the representation can propagate to downstream predictions. Prediction-space regularization acts more directly on task outputs, whereas latent-space regularization acts more directly on the representation geometry used by the downstream probe, and neither is uniformly best across the reported benchmarks. Moreover, lower martingale violation is systematically associated with better downstream robustness after accounting for the general effect of increasing completeness, as seen in negative Spearman correlations; Appendix~\ref{app:theory:bound} provides a theoretical explanation.

\begin{figure}[b!]
\centering
\includegraphics[width=1.0\linewidth]{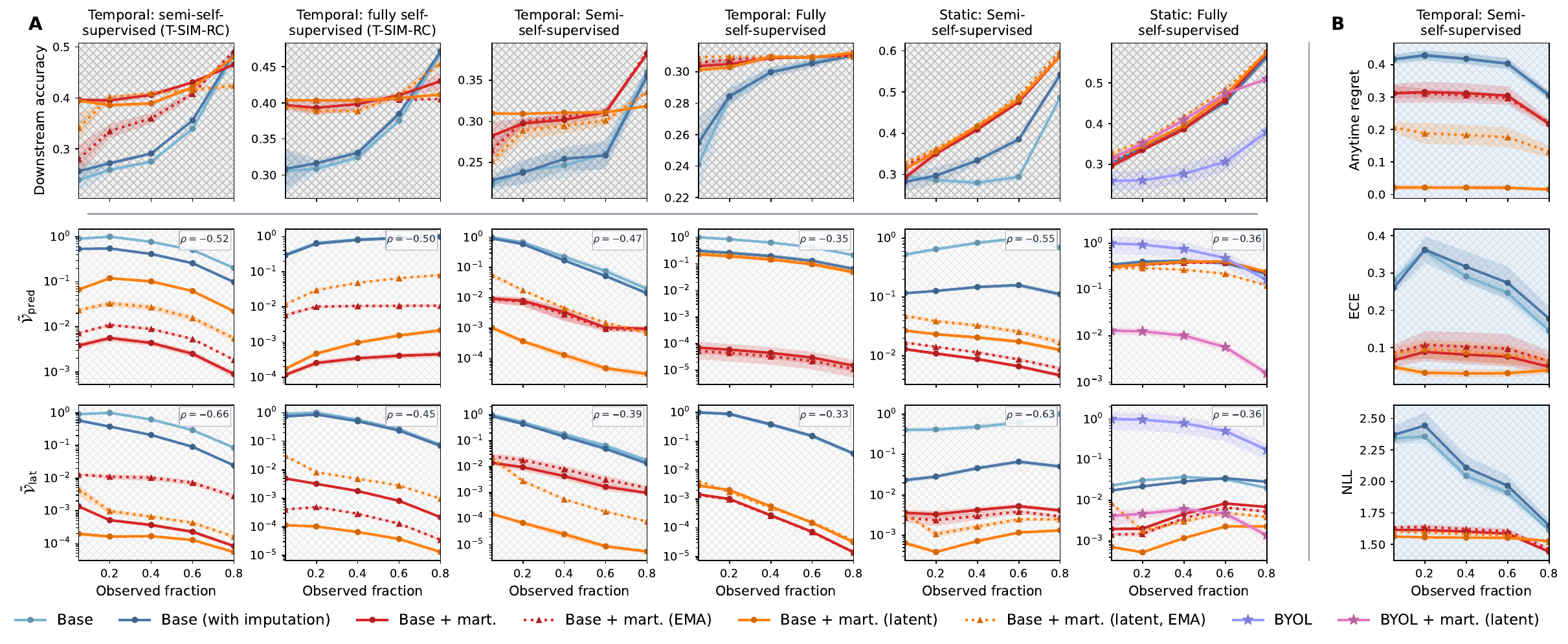}
\vspace{-0.5cm}
    \caption{(A) Robustness under missingness on simulated benchmarks. Rows show downstream probe accuracy, prediction-space violation, and latent-space violation versus completeness; violations are normalized per column. Insets show completeness-adjusted Spearman associations between downstream accuracy and log-violation. (B) calibration diagnostics for the right-censored semi-self-supervised setting–anytime regret (accuracy deficit vs.\ full-information), ECE, and NLL.}
  \label{fig:probe_curves_selected}
\end{figure}

\paragraph{Calibration under information change.} Panel B in Figure~\ref{fig:probe_curves_selected} shows calibration diagnostics for the temporal semi-self-supervised benchmark. Anytime regret (accuracy deficit relative to full information), expected calibration error (ECE), and negative log-likelihood (NLL) generally worsen when less information is available, but \textit{Base} deteriorates far more than martingale-consistent training. Across  $c\in \{0.05, 0.2, 0.4, 0.6, 0.8\}$, anytime regret averages about $0.39$ for \textit{Base} versus about $0.02$ for \textit{Base~+~mart.~(latent)}. Under partial observation, average ECE is about $0.26$ for \textit{Base} versus about $0.04$ for \textit{Base~+~mart.~(latent)}; the corresponding average increase in NLL relative to the full view is about $1.09$ versus about $0.09$. This supports the conditional expectation family view: enforcing cross-view coherence during training improves calibration under partial observation.

\vspace{-0.15cm}
\paragraph{Image benchmarks.} 
Figure~\ref{fig:image_missingness_and_performance} extends the partial-observation setting to CIFAR-10 and STL-10 under center-biased patch masking. On CIFAR-10, in the semi-self-supervised image setting, latent martingale regularization improves mean downstream accuracy from $0.311$ to $0.383$ ($+23.2\%$ over the base model; $+8.5\%$ over base with imputation) while reducing both prediction- and latent-space martingale violations. The same pattern appears for CIFAR-10 SimCLR, where the martingale variant improves mean accuracy from $0.213$ to $0.324$ ($+52.1\%$ over SimCLR; $+20.4\%$ over SimCLR with imputation) and lowers both violations. On STL-10, the semi-self-supervised martingale variant improves over the base model from $0.255$ to $0.311$ ($+22.0\%$), though the non-martingale imputation baseline reaches $0.324$. For STL-10 SimCLR, martingale regularization improves mean accuracy from $0.210$ to $0.306$ ($+45.7\%$ over SimCLR; $+9.3\%$ over SimCLR with imputation) and reduces both prediction- and latent-space violations. This suggests that the coherence constraint remains useful in higher-dimensional visual settings as well. We provide additional experimental results, including a sensitivity analysis, in Appendix~\ref{app:additional_experiments}.

\begin{figure*}[t]
  \vspace{-0cm}
  \centering
  \def\ImgOverlayTabShiftX{-4.35cm}
  \def\ImgOverlayTabShiftY{-0.28cm}
  \begin{tikzpicture}
    \node[inner sep=0](Img){\includegraphics[width=\linewidth]{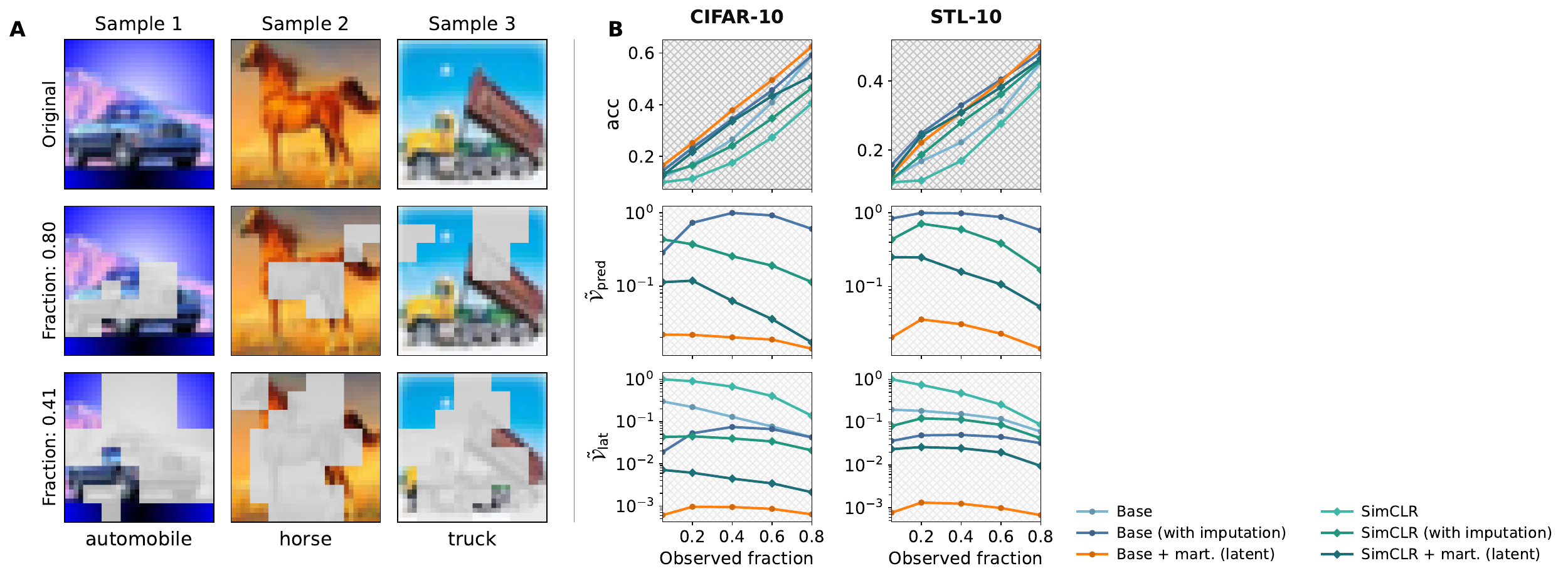}};
    \coordinate(TabAt)at($(Img.north east)+(\ImgOverlayTabShiftX,\ImgOverlayTabShiftY)$);
    \node[anchor=north west, inner sep=2.2pt, fill=white, fill opacity=0.96, text opacity=1]at(TabAt){
      \scalebox{0.87}{
    \begingroup\scriptsize\setlength{\tabcolsep}{2.5pt}\renewcommand{\arraystretch}{0.92}
      \begin{tabular}{>{\raggedright\arraybackslash}p{2.35cm}cc}
        \toprule
        & \textbf{CIFAR-10} & \textbf{STL-10} \\
        \midrule
        \multicolumn{3}{c}{\textit{Semi-self-supervised}} \\
        \midrule
        Base & 0.311 & 0.255 \\
        Base (with imp.) & 0.353 & \textbf{0.324} \\
        \rowcolor{Accent15} Base + mart. (latent) & \textbf{0.383} & 0.311 \\
        \relgaindash{1-3}
        \rowcolor{GainRowGray}\textit{Rel.\ gain (vs Base)} & \textcolor{GainGreen}{+23.2\%} & \textcolor{GainGreen}{+22.0\%} \\
        \doublemidrule
        \multicolumn{3}{c}{\textit{Fully self-supervised}} \\
        \midrule
        SimCLR & 0.213 & 0.210 \\
        SimCLR (with imp.) & 0.269 & 0.280 \\
        \rowcolor{Accent15} SimCLR+mart. (latent) & \textbf{0.324} & \textbf{0.306} \\
        \relgaindash{1-3}
        \rowcolor{GainRowGray}\textit{Rel.\ gain (vs  SimCLR)} & \textcolor{GainGreen}{+52.1\%} & \textcolor{GainGreen}{+45.7\%} \\
        \bottomrule
      \end{tabular}
      \endgroup}}
    ;
  \end{tikzpicture}
  \vspace{-0.4cm}
  \caption{(A)~CIFAR-10 partial-observation visualization using center-biased masking. (B)~Per-method mean downstream accuracy and normalized prediction- and latent-space martingale violation curves for CIFAR-10 and STL-10, together with a compact table of the reported mean accuracies}
\label{fig:image_missingness_and_performance}
\vspace{-0.2cm}
\end{figure*}

\vspace{-0.1cm}
\section{Discussion}
\label{sec:discussion}
\vspace{-0.2cm}
In this work, we introduced martingale consistency as a coherence constraint for self-supervised learning under partial information. Rather than enforcing invariance across views, the proposed objective enforces agreement in expectation under information refinement, allowing predictions to update as new information is revealed while preventing systematic drift. This provides a probabilistic perspective on SSL models deployed across varying masking patterns, context lengths, or observation \nolinebreak regimes.

Empirically, martingale regularization consistently improved robustness and calibration under partial observation across time-series, tabular, and image benchmarks, including contrastive and bootstrap-based objectives. The gains were strongest in low-information regimes, where coherent updating across nested information sets is most critical. Moreover, improvements frequently coincided with reduced martingale violations, suggesting that coherence under refinement captures a useful structural prior absent from standard SSL objectives.

The framework integrates cleanly into existing SSL pipelines through prediction- or latent-space penalties together with an unbiased two-sample estimator based on stochastic refinement. At the same time, several limitations remain. The approach depends on the quality of the refinement mechanism used to approximate conditional expectations, and the interaction between martingale objectives, architecture, and model scale warrants further investigation. More efficient or adaptive strategies for constructing refinement distributions may also improve scalability in large-scale settings.

A particularly promising direction is foundation-model pretraining, where a single model is routinely queried under varying context lengths, masking patterns, or incomplete observations. These systems are therefore implicitly expected to produce coherent updates as information changes, but existing objectives do not explicitly enforce this property. We view martingale consistency as a principled way to address this gap.

\FloatBarrier

\section*{Acknowledgements}
MG is supported by the EPSRC Centre for Doctoral Training in Health Data Science (EP/S02428X/1). CY was supported by a UKRI Turing AI Acceleration Fellowship (EP/V023233/2). The results shown in the Appendix are in part based upon data generated by the TCGA Research Network: \url{https://www.cancer.gov/tcga}.

\bibliographystyle{plainnat}
\bibliography{bibliography}

\newpage

\section*{\Large Appendix}

\appendix

\section{Impact Statement}
\label{app:impact}
This paper introduces a framework for training self-supervised models that produce coherent predictions as information is progressively revealed. By improving robustness and calibration under partial observation, this work has the potential to support more reliable downstream use in settings where data arrives gradually or unevenly, including clinical risk assessment, streaming analysis, and budgeted data acquisition. While the resulting predictions can inform decision-making in such domains, they should complement – not replace – expert judgment, conventional uncertainty quantification, and validation against the specific data-generating processes encountered at deployment. In this sense, ethical considerations include transparency about the missingness regimes under which the method has been validated, since stronger benchmark robustness may encourage use under deployment-time missingness patterns that differ qualitatively from training. One benefit of training under randomized partial-observation regimes is reduced sensitivity to any single observation pattern, and therefore better generalization across heterogeneous data-collection settings. Nevertheless, as with any method targeting partial-information regimes, care must be taken to ensure equitable benefits across populations, since missingness in real-world data is rarely uniform and can correlate with sensitive attributes. This work primarily seeks to improve coherence properties of self-supervised representations, with no foreseeable direct societal harm.

\section{Theoretical Complements}
\label{app:theoretical_complements}

\subsection{Proof of unbiasedness of the 2-sample martingale loss estimator} 
\label{app:proof_unbiasedness}
We show that the two-sample Monte Carlo estimator $\widehat{\mathcal{L}}_{\mathrm{mart}}^{2\mathrm{MC}}$ is an unbiased estimator of the conditional squared error $\|u - E[g_\phi(X_{\mathcal{F}_2}) \mid x_{\mathcal{F}_1}]\|^2$.
\begin{proof}
Fix $x_{\mathcal{F}_1}$ and write
$\mu := E\left[ g_\phi(X_{\mathcal{F}_2}) \mid x_{\mathcal{F}_1} \right]$.
Assume $v_a$ and $v_b$ are drawn conditionally independent given $x_{\mathcal{F}_1}$ and that all vectors have finite second moments. Expanding $(u-v_a)^\top(u-v_b)$ yields $\widehat{\mathcal{L}}_{\mathrm{mart}}^{2\mathrm{MC}}=u^\top u - u^\top v_b - v_a^\top u + v_a^\top v_b$.
Taking the conditional expectation given $x_{\mathcal{F}_1}$ and using linearity, conditional independence, and $E[v_a\!\mid\!x_{\mathcal{F}_1}]\!=\! E[v_b\!\mid\! x_{\mathcal{F}_1}]\!=\!\mu$, we obtain
\begin{align*}
E\left[ \widehat{\mathcal{L}}_{\mathrm{mart}}^{2\mathrm{MC}} \mid x_{\mathcal{F}_1} \right]
&= u^\top u - u^\top E[v_b \mid x_{\mathcal{F}_1}] - E[v_a \mid x_{\mathcal{F}_1}]^\top u + E\left[ v_a^\top v_b \mid x_{\mathcal{F}_1} \right] \\
&= \|u\|^2 - u^\top \mu - \mu^\top u + E[v_a \mid x_{\mathcal{F}_1}]^\top E[v_b \mid x_{\mathcal{F}_1}] \\
&= \|u\|^2 - 2u^\top \mu + \|\mu\|^2 \\
&= \|u - \mu\|^2.
\end{align*}
\end{proof}

\subsection{Extension to non-nested information sets}
The reported experiments do not use non-nested training. In all reported experiments, the coarse view is partially observed and the refined view is always the complete input, so the empirical pipeline is nested throughout. This is the most natural setup for benchmarks where information is progressively revealed, and the one implemented across all training and evaluation pipelines.

Nevertheless, the same coherence perspective suggests a natural extension beyond nested information structures. If $\mathcal{F}_1$ and $\mathcal{F}_2$ are not ordered by inclusion but share some overlap, define
\begin{equation}
\mathcal{F}_\cap := \mathcal{F}_1 \cap \mathcal{F}_2.
\end{equation}
Writing
\begin{equation}
m_{\mathcal{F}} := E[Y \mid \mathcal{F}],
\end{equation}
the corresponding population consistency relation is
\begin{equation}
E[m_{\mathcal{F}_1}\mid \mathcal{F}_\cap] = E[m_{\mathcal{F}_2}\mid \mathcal{F}_\cap].
\end{equation}
Both sides equal $E[Y\mid\mathcal{F}_\cap]$ by the tower property (since $\mathcal{F}_\cap\subseteq\mathcal{F}_i$), so the identity holds at the population level. This does not hold for an arbitrary trained predictor, which is why it can serve as a regularizer in the non-nested case. When the views are nested, $\mathcal{F}_\cap=\mathcal{F}_1$, so this reduces to the usual martingale identity.

This extension is most relevant for settings with overlapping but incomparable views, such as heterogeneous feature subsets or partially overlapping temporal windows. It is less relevant for the partial-observation experiments in this paper, because those benchmarks are explicitly framed as refinement from a partial view to full observation.

\subsection{Excess risk bound under martingale violation}
\label{app:theory:bound}

The following result formalizes the empirical association reported in Section~\ref{sec:experiments}: the coarse-level excess risk decomposes into the martingale violation and the fine-level error, providing a theoretical explanation for why lower violation is associated with better robustness.

\begin{proposition}[Excess risk decomposition]
\label{prop:excess_risk}
Let $\mathcal{F}_1 \subseteq \mathcal{F}_2$ and let $g_\phi$ be any predictor with finite second moments. Define the fine-level excess risk $R(\mathcal{F}_2;\,g_\phi) := E[\|g_\phi(x_{\mathcal{F}_2}) - E[Y|\mathcal{F}_2]\|^2]$ and the martingale violation $\mathcal{V}(g_\phi) := E[\|g_\phi(x_{\mathcal{F}_1}) - E[g_\phi(x_{\mathcal{F}_2})|\mathcal{F}_1]\|^2]$. Then
\begin{equation}
E\bigl[\|g_\phi(x_{\mathcal{F}_1}) - E[Y|\mathcal{F}_1]\|^2\bigr] \;\leq\; 2\,\mathcal{V}(g_\phi) + 2\,R(\mathcal{F}_2;\,g_\phi).
\end{equation}
\end{proposition}

\begin{proof}
By the tower property, $E[E[Y|\mathcal{F}_2]|\mathcal{F}_1] = E[Y|\mathcal{F}_1]$. Decompose
\begin{equation*}
g_\phi(x_{\mathcal{F}_1}) - E[Y|\mathcal{F}_1]
= \underbrace{\bigl(g_\phi(x_{\mathcal{F}_1}) - E[g_\phi(x_{\mathcal{F}_2})|\mathcal{F}_1]\bigr)}_{A}
+ \underbrace{E\bigl[g_\phi(x_{\mathcal{F}_2}) - E[Y|\mathcal{F}_2]\mid\mathcal{F}_1\bigr]}_{B}.
\end{equation*}
Applying $\|a+b\|^2 \leq 2\|a\|^2 + 2\|b\|^2$ and taking expectations gives $E[\|A\|^2] = \mathcal{V}(g_\phi)$. For term $B$, Jensen's inequality yields $\|B\|^2 \leq E[\|g_\phi(x_{\mathcal{F}_2})-E[Y|\mathcal{F}_2]\|^2|\mathcal{F}_1]$, whose expectation equals $R(\mathcal{F}_2;\,g_\phi)$ by the tower property.
\end{proof}

When the model fits the refined view well ($R(\mathcal{F}_2;\,g_\phi)$ small), reducing $\mathcal{V}(g_\phi)$ is the decisive lever for coarse-level accuracy. In the limiting case $R(\mathcal{F}_2;\,g_\phi)=0$, the bound reduces to $E[\|g_\phi(x_{\mathcal{F}_1})-E[Y|\mathcal{F}_1]\|^2] \leq 2\,\mathcal{V}(g_\phi)$.

\subsection{Linear-Gaussian closed form}
\label{app:theory:linear}

To give geometric intuition for the inductive bias, we characterize the minimizer of the joint objective in a linear-Gaussian setting.

\paragraph{Setup.} Let $X = (X_1^\top, X_2^\top)^\top \sim \mathcal{N}(0, \Sigma)$ with $\Sigma = \bigl(\begin{smallmatrix}\Sigma_{11}&\Sigma_{12}\\\Sigma_{21}&\Sigma_{22}\end{smallmatrix}\bigr)$, and $Y = w_1^\top X_1 + w_2^\top X_2 + \varepsilon$, $\varepsilon\sim\mathcal{N}(0,\sigma^2)$ independent. The coarse view observes $X_1$ and the refined view observes the full $X$. For linear predictors $g_\beta(x)=\beta_1^\top x_1+\beta_2^\top x_2$ applied to whichever components are observed, the martingale loss simplifies to $\mathcal{L}_{\mathrm{mart}}(\beta) = \beta_2^\top C\,\beta_2$ where $C := \Sigma_{21}\Sigma_{11}^{-1}\Sigma_{12} \succeq 0$.

\begin{proposition}[Linear-Gaussian minimizer]
\label{prop:linear_gaussian}
Let $\Sigma_{2|1} := \Sigma_{22} - C$. The unique minimizer of $\mathcal{L}_{\mathrm{pred}}(\beta) + \lambda\,\mathcal{L}_{\mathrm{mart}}(\beta)$ is
\begin{equation}
\beta_2^*(\lambda) = \bigl(\Sigma_{2|1} + \lambda C\bigr)^{-1}\Sigma_{2|1}\,w_2,
\qquad
\beta_1^*(\lambda) = w_1 + \Sigma_{11}^{-1}\Sigma_{12}\bigl(w_2 - \beta_2^*(\lambda)\bigr).
\end{equation}
\end{proposition}

\begin{proof}
The predictive loss is $(\beta-w)^\top\Sigma(\beta-w)+\sigma^2$. Setting $\nabla_{\beta_1}=0$ gives $\beta_1^*=w_1-\Sigma_{11}^{-1}\Sigma_{12}(\beta_2^*-w_2)$. Substituting into $\nabla_{\beta_2}=0$: the cross term $2\Sigma_{21}(\beta_1^*-w_1)$ becomes $-2\Sigma_{21}\Sigma_{11}^{-1}\Sigma_{12}(\beta_2^*-w_2)=-2C(\beta_2^*-w_2)$, leaving $-2C(\beta_2^*-w_2)+2\Sigma_{22}(\beta_2^*-w_2)+2\lambda C\beta_2^*=0$. Using $\Sigma_{22}-C=\Sigma_{2|1}$ yields $(\Sigma_{2|1}+\lambda C)\beta_2^*=\Sigma_{2|1}w_2$.
\end{proof}

At $\lambda=0$, $\beta^*(0)=w$ (OLS). As $\lambda\to\infty$, $\beta_2^*\to 0$ and $\beta_1^*\to w_1+\Sigma_{11}^{-1}\Sigma_{12}w_2$, which is exactly the coefficient of the optimal coarse predictor $E[Y|X_1]$. For finite $\lambda$, martingale regularization shrinks $\beta_2^*$ along directions of highest correlation with $X_1$ (eigenstructure of $C$) and adjusts $\beta_1^*$ to compensate, progressively recovering the optimal conditional expectation at the coarse level.

\section{Additional Experiments and Results}
\label{app:additional_experiments}

\subsection{Lambda sensitivity study}
\label{app:additional_experiments:sensitivity}
\begin{figure}[b!]
  \centering
  \includegraphics[width=\linewidth]{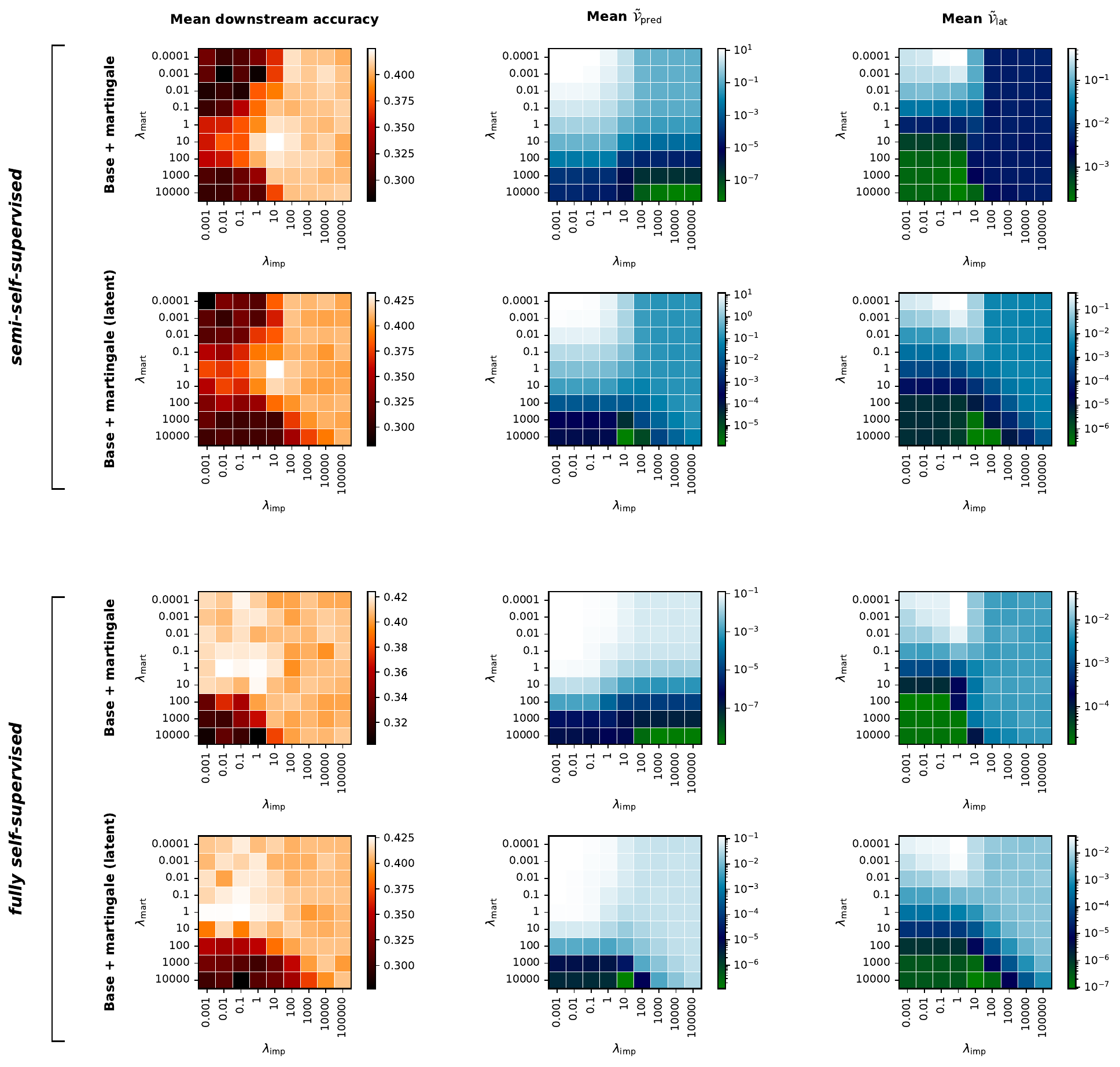}
  \caption{Lambda sensitivity study on S-SIM.
  Top block: Semi-self-supervised setting.
  Bottom block: fully self-supervised setting.
  Within each block, rows summarize Base SSL models with martingale regularization in prediction- and latent-space; columns show mean downstream accuracy, mean prediction-space martingale violation, and mean latent-space martingale violation. Values are averaged across 3 runs and completeness levels $c\in\{0.8,0.6,0.4,0.2,0.05\}$.}
  \label{fig:lambda_sensitivity_static_combined}
\end{figure}

To study hyperparameter sensitivity, we perform a sensitivity study over
$(\lambda_{\mathrm{imp}},\lambda_{\mathrm{mart}})$ for the static tabular benchmark on the S-SIM dataset, in both semi-self-supervised and fully self-supervised settings.
Results are summarized in Fig.~\ref{fig:lambda_sensitivity_static_combined}.

Across both settings, the strongest downstream-accuracy regions are broadly similar: they concentrate around moderate imputation weight
($\lambda_{\mathrm{imp}}\approx 0.1$--$1$ in the fully self-supervised setting and around $\lambda_{\mathrm{imp}}\approx 10$ in the semi-self-supervised setting)
and moderate martingale weight ($\lambda_{\mathrm{mart}}\approx 1$--$10$).
Very small $\lambda_{\mathrm{imp}}$ values are comparatively less favorable in both settings, while larger $\lambda_{\mathrm{imp}}$ can partly compensate when $\lambda_{\mathrm{mart}}$ is small.

For martingale-violation panels, the dominant trend is along $\lambda_{\mathrm{mart}}$: increasing $\lambda_{\mathrm{mart}}$ drives both prediction-space and latent-space violations down, often several orders of magnitude across the sweep.
The effect of $\lambda_{\mathrm{imp}}$ is conditional on $\lambda_{\mathrm{mart}}$ rather than monotone on its own.
At low $\lambda_{\mathrm{mart}}$, violations tend to be highest around intermediate $\lambda_{\mathrm{imp}}$ values (often near $1$), whereas at high $\lambda_{\mathrm{mart}}$ the violations are low across a much broader $\lambda_{\mathrm{imp}}$ range.
For the prediction-space martingale regularization, higher $\lambda_{\mathrm{imp}}$ together with strong martingale regularization also tends to further reduce prediction-space violation, consistent with improved completion quality yielding more coherent refined predictions.

\subsection{Frozen martingale estimator bias diagnostic}
\label{app:additional_experiments:estimator_bias}
We compare two finite-sample estimators of the martingale penalty against a large-$K$ Monte Carlo reference. For a frozen model $g_\phi$, coarse view $x_{\mathcal{F}_1}$, and $u := g_\phi(x_{\mathcal{F}_1})$, the ideal martingale squared error is $L := \|u - \mu\|_2^2$ with $\mu := E[g_\phi(x_{\mathcal{F}_2})\mid x_{\mathcal{F}_1}]$. We approximate $\mu$ by a $K\!=\!128$ Monte Carlo average $\hat{\mu}$ and set $L_{\mathrm{ref}} := \|u-\hat{\mu}\|_2^2$. The two candidate estimators are the naive single-sample plug-in
\begin{equation}
\widehat{L}_{\mathrm{single}} := \left\|u-v_a\right\|_2^2,
\end{equation}
and the two-independent-sample construction
\begin{equation}
\widehat{L}_{\mathrm{two}} := \left(u-v_a\right)^\top\left(u-v_b\right),
\end{equation}
where $v_a, v_b$ are conditionally independent refinement samples. We train two model families on S-SIM---one with each estimator---and evaluate each frozen model with its native estimator; the diagnostic value is $|\hat{L}-L_{\mathrm{ref}}|$ averaged over test batches. This is repeated for both training settings and two refinement mechanisms: the learned imputer and an oracle conditional sampler from the simulated DGP (Figure~\ref{fig:estimator_bias_both_worlds}).

Across both refinement mechanisms, the two-sample estimator tracks the reference more closely than the naive plug-in, with the gap largest at lower completeness. The single-sample estimator is biased upward by $\mathrm{Var}(v_a\mid x_{\mathcal{F}_1})$, which grows when completeness is low; the two-sample construction eliminates this bias. The smaller discrepancy $|\hat{L}-L_{\mathrm{ref}}|$ therefore reflects bias reduction, not variance reduction.
\begin{figure}[h!]
  \vspace{0.3cm}
  \centering
  \includegraphics[width=\linewidth]{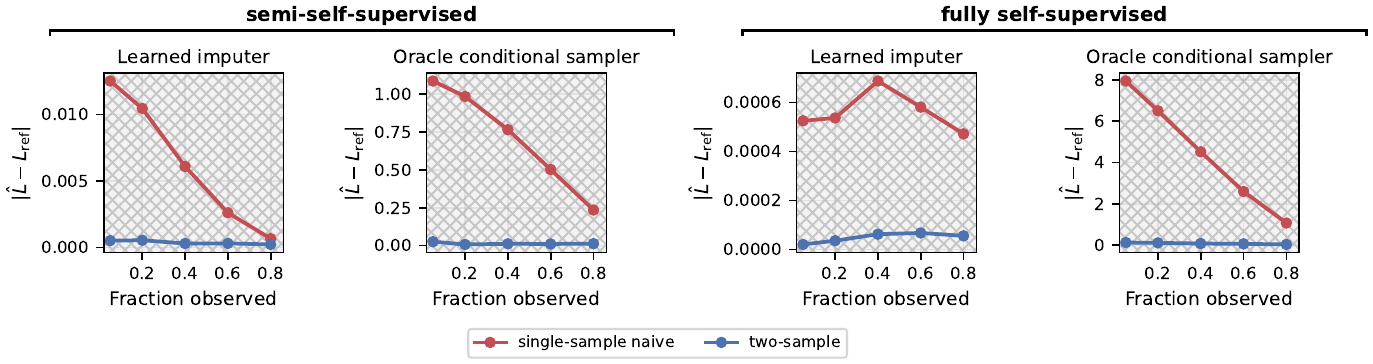}
  \caption{Frozen martingale estimator bias diagnostic.
  Results are shown for the static tabular benchmark on the S-SIM dataset averaged over 5 runs.
  The x-axis is the observed fraction of the coarse view.
  The y-axis is the absolute discrepancy $|\hat{L}-L_{\mathrm{ref}}|$ between a finite-sample martingale error estimator $\hat{L}$
  (naive single-sample vs two-sample) and a $K\!=\!128$ Monte Carlo reference $L_{\mathrm{ref}}$.
  Panels are split by training setting (semi-self-supervised vs fully self-supervised) and by the refinement sampling mechanism
  (learned imputer vs oracle conditional sampler)}
  \label{fig:estimator_bias_both_worlds}
\end{figure}

\subsection{Computational overhead}
\label{app:additional_experiments:compute_overhead}
Figure~\ref{fig:compute_overhead_hist} shows that overhead increases are generally small across runtime, peak GPU memory, and parameter count. Parameter growth is particularly limited, which is consistent with the slim implementation of the additional imputation component. Runtime and memory rise only mildly in most comparisons, indicating that the added regularization remains lightweight in practice.

\begin{figure}[h!]
  \centering
  \includegraphics[width=\linewidth]{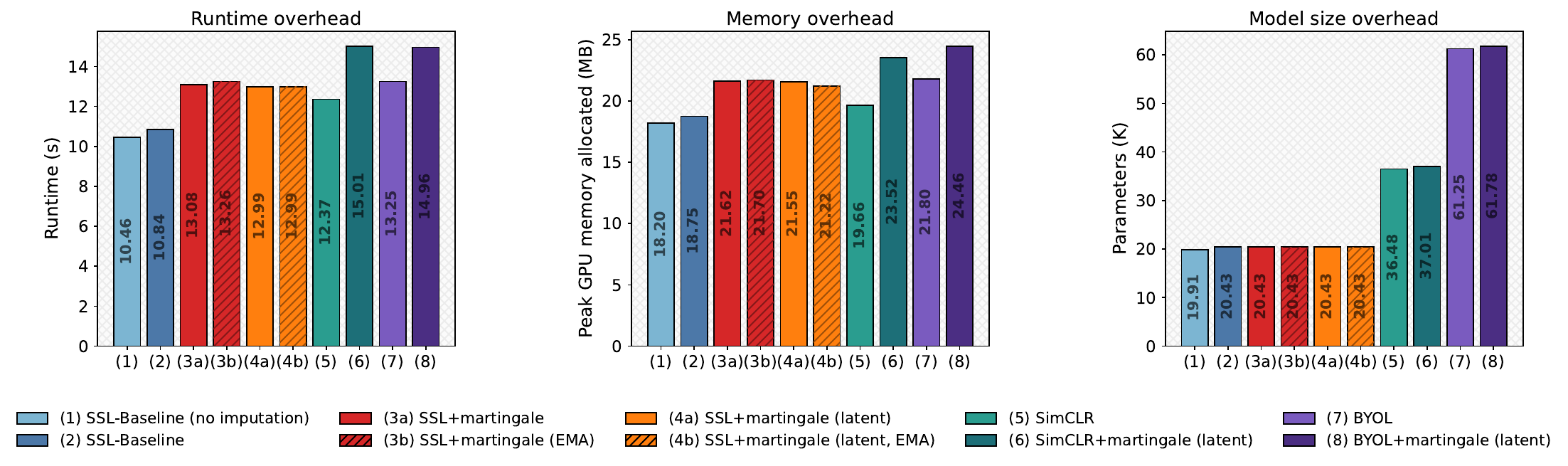}
  \caption{Computational overhead comparison across methods. From left to right, panels report runtime, peak GPU memory allocation, and total parameter count.}
  \label{fig:compute_overhead_hist}
\end{figure}

\FloatBarrier

\section{TCGA Multi-Omics Survival Benchmark}
\label{app:tcga}

We additionally test martingale-consistency regularization in a biomedical setting with intrinsic missingness. The benchmark uses per-cancer cohorts from The Cancer Genome Atlas (TCGA), where patients are represented by multiple omics modalities, including RNA expression, DNA methylation, copy-number variation, reverse-phase protein arrays (RPPA), and miRNA when available. Unlike the controlled missingness experiments in the main text, TCGA patients are not fully observed and are typically measured only on a patient-specific subset of modalities. We therefore use this experiment as a complementary stress test: the missingness pattern used at evaluation time is the cohort's native modality availability, not an artificially simulated test mask.

The base model is MIND \cite{Xing2025_MIND}, a multimodal variational autoencoder for integrating partially observed multi-omics profiles into a shared patient-level representation. For each cancer type, we train a MIND model, extract the posterior-mean patient embeddings, and fit a Coxnet survival model \cite{Pölsterl2020} using those embeddings together with age and gender covariates. Performance is measured by the C-index from 5-fold cross-validation. We report average results across 3 random-seed repetitions per cancer and variant.

This setting changes the interpretation of the refined information set. In the main experiments, the refined view $\mathcal{F}_2$ is the fully observed input. For TCGA, the fully observed multimodal profile is not available for most patients. We therefore define $\mathcal{F}_2=\mathcal{F}_{M_i}$ as the patient's maximally available view, where $M_i$ is the set of modalities observed for patient $i$. During training, for patients with $|M_i|=n_i\geq2$, we sample $k_i\sim\mathrm{Uniform}\{1,\ldots,n_i-1\}$ and then sample a uniform subset $M_i'\subset M_i$ of size $k_i$, yielding a nested coarse view $\mathcal{F}_{M_i'}\subset\mathcal{F}_{M_i}$. Patients with a single available modality contribute to the base MIND objective but are excluded from the martingale reduction, since the proper-coarsening term is degenerate. Thus this experiment asks whether martingale regularization improves representation learning under further coarsening of already partially observed real data.

We train six models: \emph{MIND}; \emph{MIND (with imputation)}, which uses the auxiliary imputation loss but not the martingale term; and four martingale-regularized variants with prediction- or latent-space losses, with and without EMA targets. MIND's modality-specific encoders and decoders provide the imputation mechanism used to form two independent stochastic refinements. We set $\lambda_{\mathrm{imp}}=\lambda_{\mathrm{mart}}=1$ and use EMA decay $0.999$ for the EMA variants.

\paragraph{Results.}
Table~\ref{tab:tcga_survival} reports wins against \emph{MIND}: for each method, the final row counts how many of the $17$ cohorts have a strictly higher C-index than \emph{MIND}. \emph{MIND (with imputation)} does so in $6$ cohorts. The four martingale variants do so in $8$--$13$ cohorts, with prediction-space EMA ($13$) most often, then prediction-space ($12$) and latent-space EMA ($12$), and latent space without EMA ($8$). 

Overall, the TCGA study supports the same qualitative conclusion as the main benchmarks in a different regime. Martingale consistency is not being used here to refine a simulated partial view into a fully observed sample; instead, it regularizes the coherence between a further-coarsened subset of the patient's observed modalities and that patient's own maximally available multimodal profile. These win rates suggest that martingale regularization is often helpful under intrinsic missingness, while the spread across variants underscores that the best martingale space and target-network choice can be application dependent.

\begin{table*}[t!]
  \centering
  \footnotesize
  \caption{TCGA multi-omics survival benchmark. Entries are 5-fold Coxnet C-index values computed from MIND patient embeddings, shown as mean $\pm$ SEM over three training seeds. The final row counts cohort wins against \emph{MIND} (strictly higher C-index).}
  \label{tab:tcga_survival}
  \begingroup
  \setlength{\tabcolsep}{3.5pt}
  \setlength{\extrarowheight}{1.5pt}
  \begin{tabularx}{\textwidth}{l P R Q Q Q Q}
    \toprule
    \rowcolor{GainRowGray}
    Cancer & MIND & \makecell[tc]{MIND (with\\imputation)} & \makecell[tc]{MIND + mart.} & \makecell[tc]{MIND + mart.\\(EMA)} & \makecell[tc]{MIND + mart.\\(latent)} & \makecell[tc]{MIND + mart.\\(latent, EMA)} \\
    \midrule
    BLCA & 0.6185{\taberr{.0039}} & 0.6213{\taberr{.0035}} & \tcgamart{\textbf{0.6247{\taberr{.0048}}}} & \tcgamart{0.6233{\taberr{.0048}}} & \tcgamart{0.6121{\taberr{.0047}}} & \tcgamart{0.5872{\taberr{.0025}}} \\
    BRCA & 0.6409{\taberr{.0077}} & \textbf{0.6507{\taberr{.0050}}} & \tcgamart{0.6358{\taberr{.0053}}} & \tcgamart{0.6448{\taberr{.0050}}} & \tcgamart{0.6407{\taberr{.0017}}} & \tcgamart{0.6149{\taberr{.0033}}} \\
    CESC & 0.6191{\taberr{.0101}} & 0.5654{\taberr{.0053}} & \tcgamart{0.6138{\taberr{.0242}}} & \tcgamart{0.5966{\taberr{.0079}}} & \tcgamart{0.6119{\taberr{.0056}}} & \tcgamart{\textbf{0.6447{\taberr{.0098}}}} \\
    COAD & 0.5188{\taberr{.0066}} & 0.5376{\taberr{.0192}} & \tcgamart{0.5441{\taberr{.0210}}} & \tcgamart{0.4913{\taberr{.0047}}} & \tcgamart{0.5351{\taberr{.0109}}} & \tcgamart{\textbf{0.5739{\taberr{.0013}}}} \\
    COADREAD & 0.5780{\taberr{.0092}} & 0.5590{\taberr{.0128}} & \tcgamart{0.5837{\taberr{.0116}}} & \tcgamart{0.5855{\taberr{.0054}}} & \tcgamart{0.5732{\taberr{.0035}}} & \tcgamart{\textbf{0.6146{\taberr{.0026}}}} \\
    GBM & 0.6426{\taberr{.0045}} & 0.6449{\taberr{.0037}} & \tcgamart{0.6533{\taberr{.0014}}} & \tcgamart{0.6501{\taberr{.0023}}} & \tcgamart{\textbf{0.6564{\taberr{.0031}}}} & \tcgamart{0.6541{\taberr{.0020}}} \\
    HNSC & 0.6263{\taberr{.0062}} & 0.6240{\taberr{.0002}} & \tcgamart{0.6009{\taberr{.0018}}} & \tcgamart{0.6039{\taberr{.0048}}} & \tcgamart{0.6160{\taberr{.0012}}} & \tcgamart{\textbf{0.6324{\taberr{.0011}}}} \\
    KIRC & 0.7102{\taberr{.0031}} & 0.7097{\taberr{.0029}} & \tcgamart{\textbf{0.7199{\taberr{.0026}}}} & \tcgamart{0.7167{\taberr{.0017}}} & \tcgamart{0.7028{\taberr{.0007}}} & \tcgamart{0.7065{\taberr{.0012}}} \\
    LGG & 0.8422{\taberr{.0026}} & 0.8394{\taberr{.0016}} & \tcgamart{0.8476{\taberr{.0004}}} & \tcgamart{\textbf{0.8494{\taberr{.0011}}}} & \tcgamart{0.8488{\taberr{.0010}}} & \tcgamart{0.8472{\taberr{.0027}}} \\
    LIHC & 0.6063{\taberr{.0109}} & 0.6043{\taberr{.0115}} & \tcgamart{\textbf{0.6401{\taberr{.0011}}}} & \tcgamart{0.6236{\taberr{.0025}}} & \tcgamart{0.6352{\taberr{.0071}}} & \tcgamart{0.6100{\taberr{.0243}}} \\
    LUAD & 0.6045{\taberr{.0043}} & 0.5905{\taberr{.0056}} & \tcgamart{0.5834{\taberr{.0031}}} & \tcgamart{\textbf{0.6102{\taberr{.0153}}}} & \tcgamart{0.5965{\taberr{.0036}}} & \tcgamart{0.5529{\taberr{.0015}}} \\
    LUSC & 0.5253{\taberr{.0088}} & 0.5397{\taberr{.0100}} & \tcgamart{0.5414{\taberr{.0076}}} & \tcgamart{0.5386{\taberr{.0042}}} & \tcgamart{0.5398{\taberr{.0067}}} & \tcgamart{\textbf{0.5467{\taberr{.0025}}}} \\
    OV & 0.5983{\taberr{.0079}} & 0.5966{\taberr{.0048}} & \tcgamart{0.5809{\taberr{.0051}}} & \tcgamart{0.5941{\taberr{.0040}}} & \tcgamart{0.5977{\taberr{.0033}}} & \tcgamart{\textbf{0.5994{\taberr{.0007}}}} \\
    SKCM & 0.6461{\taberr{.0032}} & 0.6364{\taberr{.0028}} & \tcgamart{0.6522{\taberr{.0011}}} & \tcgamart{0.6631{\taberr{.0003}}} & \tcgamart{\textbf{0.6641{\taberr{.0013}}}} & \tcgamart{0.5994{\taberr{.0028}}} \\
    STAD & 0.5619{\taberr{.0093}} & 0.5635{\taberr{.0116}} & \tcgamart{0.5749{\taberr{.0090}}} & \tcgamart{\textbf{0.5788{\taberr{.0057}}}} & \tcgamart{0.5785{\taberr{.0031}}} & \tcgamart{0.5730{\taberr{.0019}}} \\
    THCA & 0.9336{\taberr{.0054}} & 0.9326{\taberr{.0067}} & \tcgamart{0.9438{\taberr{.0008}}} & \tcgamart{\textbf{0.9515{\taberr{.0018}}}} & \tcgamart{0.9303{\taberr{.0022}}} & \tcgamart{0.9479{\taberr{.0028}}} \\
    UCEC & 0.6281{\taberr{.0086}} & 0.6259{\taberr{.0065}} & \tcgamart{0.6483{\taberr{.0108}}} & \tcgamart{\textbf{0.6509{\taberr{.0121}}}} & \tcgamart{0.6413{\taberr{.0034}}} & \tcgamart{0.6491{\taberr{.0038}}} \\
    \relgaindash{1-7}
    \rowcolor{GainRowGray}\textit{Wins / 17} & -- & 6 & 12 & 13 & 8 & 12 \\
    \bottomrule
  \end{tabularx}
  \endgroup
\end{table*}

\section{Data Descriptions}
\label{app:data}

\subsection{Simulated Data}
\paragraph{Synthetic time-series data (right-censored \emph{T-SIM-RC}).}
For the right-censored benchmark, we generate synthetic multivariate time series from a causal linear dynamical system with a bounded nonlinearity,
\begin{equation}
X_0 \sim \mathcal{N}(0, \Sigma_0),\qquad
X_t = A X_{t-1} + \varepsilon_t + \alpha\,\tanh\bigl(A X_{t-1} + \varepsilon_t\bigr), \quad t = 1,\ldots,T,
\end{equation}
where $\varepsilon_t \sim \mathcal{N}(0,\sigma^2 I_D)$ are i.i.d., $A \in \mathbb{R}^{D \times D}$ is a stable lower-triangular transition matrix, and $\Sigma_0$ is the stationary covariance of the linear part. In the reported setting we use $T=16$, $D=16$, noise scale $\sigma=0.2$, nonlinearity strength $\alpha=0.2$, and five classes, with labels defined as $y=\arg\max_{k \in \{1,\ldots,5\}} (W_{\mathrm{RC}} u)_k$, where $u$ is a weighted combination of the final $W=5$ time steps and $W_{\mathrm{RC}}$ is a fixed linear map from that summary vector to five logits. This construction  concentrates label information near the end of the sequence, so right-censoring removes the most predictive timesteps first.

\paragraph{Synthetic time-series data (\emph{T-SIM}).}
For the structured temporal benchmark, we use a different non-Markov generator with independent latent factors at each time step:
\begin{equation}
z_t \sim \mathcal{N}(0, I_k),\qquad
X_t = B_t z_t + \varepsilon_t + \alpha\,\tanh\bigl(B_t z_t + \varepsilon_t\bigr), \quad t = 1,\ldots,T,
\end{equation}
where $B_t \in \mathbb{R}^{D \times k}$ is a time-indexed projection, $\varepsilon_t \sim \mathcal{N}(0,\sigma^2 I_D)$, and the latent dimension is $k=8$. In the reported setting we use $T=16$, $D=32$, noise scale $\sigma=0.2$, nonlinearity strength $\alpha=0.2$, and five classes. Labels use the frames at time indices $\{8,9,10,15\}$: we form a $D$-dimensional summary by averaging those four frames coordinatewise, apply $\tanh$ to that summary coordinatewise, zero out all but $|S|=5$ feature dimensions using a fixed mask determined by the simulation seed (leaving only those coordinates informative for the label), and set $y=\arg\max_{k \in \{1,\ldots,5\}} (W_{\mathrm{ST}}\tilde{x})_k$ for a fixed random matrix $W_{\mathrm{ST}}\in\mathbb{R}^{5\times D}$ and masked summary $\tilde{x}\in\mathbb{R}^D$.

\paragraph{Synthetic static tabular dataset (\emph{S-SIM}).}
\emph{S-SIM} is the synthetic static tabular dataset used in our experiments. Each sample is a feature vector, and partial observation is induced by feature masking. Data generation uses one shared standard normal factor $g$ and independent standard normal pair factors $(f_j)_{j=1}^{|S|}$ aligned with a fixed predictive subset $S$ of coordinates. Each predictive coordinate loads on $(g,f_j)$ with fixed affine coefficients and Gaussian noise; up to $|S|$ additional \emph{proxy} coordinates load on the same factors (so predictive information can be partially recoverable when $S$ is masked); all remaining coordinates are independent Gaussian noise. We then apply additive Gaussian measurement noise with scale $\sigma$ and bounded nonlinearity strength $\alpha$, using the same $(\sigma,\alpha)$ as for \emph{T-SIM-RC} and \emph{T-SIM}. Labels are defined by a linear multiclass score on $S$, $y=\arg\max_{k \in \{1,\ldots,K\}} (W x)_k$, with class-weight matrix $W \in \mathbb{R}^{K \times D}$ supported only on $S$. In the reported setting we use $D=16$, $K=5$, and $|S|=5$.

\subsection{Real and image data}
The real-world time-series datasets are \emph{UCI HAR (HAR)} \cite{Anguita2013_UCIHAR}, \emph{Traffic (TRAF)} (PEMS-SF) \cite{Cuturi2011_PEMSSF}, \emph{Stock (STK)}\footnote{Stock histories were downloaded using the \texttt{yfinance} package for Yahoo Finance market data: \url{https://pypi.org/project/yfinance/}.}, \emph{Character Trajectories (CT)} \cite{Williams2006_CharacterTrajectories}, and \emph{Spoken Arabic Digits (SAD)} \cite{BeddaHammami2008_SpokenArabicDigit}. These are multivariate sequence-classification problems that differ substantially in sequence length, dimensionality, and domain. The stock benchmark is generated from daily OHLCV histories for a fixed set of large-cap U.S. equities, using approximately the two years of data available at dataset-construction time. We then standardize each ticker series, extract sliding windows of length 128 with stride 32, and assign a binary label indicating whether the next closing price increases.

The static datasets are \emph{Adult (AD)} \cite{BeckerKohavi1996_Adult}, \emph{Bank Marketing (BM)} \cite{MoroCortezRita2014_BankMarketing}, \emph{Credit-g (CG)} \cite{Hofmann1994_CreditG}, and \emph{Phoneme (PH)} \cite{HastieBujaTibshirani1995_Phoneme}. These are tabular classification tasks used to test whether the same martingale-consistency story remains meaningful when nested information is induced by feature revelation rather than temporal continuation. The image benchmarks are CIFAR-10 \cite{Krizhevsky2009_CIFAR10} and STL-10 \cite{Gordon2011_STL-10}, both treated as $32 \times 32$ color-image classification problems under patch masking, where the coarse view reveals only a subset of image patches and the refined view is the full image.

\subsection{Dataset summary}

\begin{table*}[b!]
  \centering
  \caption{Benchmark dataset summary.}
  \label{tab:dataset_summary}
  \begin{tabularx}{\textwidth}{@{}l r r r c X@{}}
    \toprule
    Dataset & Samples & Classes & Features & Cat./Cont. & Domain \\
    \midrule
    \multicolumn{6}{@{}l}{\textbf{Time-series benchmarks}} \\
    T-SIM-RC & 6000 & 5 & 16 & 0 / 16 & Synthetic time series \\
    T-SIM & 6500 & 5 & 32 & 0 / 32 & Synthetic time series \\
    UCI HAR (HAR) & 10299 & 6 & 9 & 0 / 9 & Activity sensing \\
    Traffic (TRAF) & 440 & 2 & 963 & 0 / 963 & Transportation \\
    Stock (STK) & 180 & 2 & 5 & 0 / 5 & Finance \\
    Character Trajectories (CT) & 2858 & 20 & 3 & 0 / 3 & Handwriting \\
    Spoken Arabic Digits (SAD) & 8800 & 10 & 13 & 0 / 13 & Speech \\
    \midrule
    \multicolumn{6}{@{}l}{\textbf{Static tabular benchmarks}} \\
    S-SIM & 6500 & 5 & 16 & 0 / 16 & Synthetic tabular \\
    Adult (AD) & 48842 & 2 & 14 & 8 / 6 & Demographics \\
    Bank Marketing (BM) & 45211 & 2 & 16 & 10 / 6 & Marketing \\
    Credit-g (CG) & 1000 & 2 & 20 & 13 / 7 & Finance \\
    Phoneme (PH) & 5404 & 2 & 5 & 0 / 5 & Speech \\
    \midrule
    \multicolumn{6}{@{}l}{\textbf{Image benchmarks}} \\
    CIFAR-10 & 60000 & 10 & 32$\!\times\!$32$\!\times\!$3 & – & Natural images \\
    STL-10 & 13000 & 10 & 32$\!\times\!$32$\!\times\!$3 & – & Natural images \\
    \bottomrule
  \end{tabularx}
\end{table*}

Table~\ref{tab:dataset_summary}
summarizes the simulated, real, and image benchmarks used in the experiments. For the simulated benchmarks, T-SIM-RC uses a 3000/3000 train/test split. T-SIM and S-SIM use 3000 training examples, 500 held-out prior-mask-fitting examples, and 3000 test examples. For the real time-series and static tabular datasets, we use a 60\% / 10\% / 30\% train / prior-mask-fit / test split. For CIFAR-10 and STL-10, we use a 70\% / 30\% train/test split with no separate prior-mask-fitting partition; STL-10 images are resized from $96 \times 96$ to $32 \times 32$ for the image experiments.

\subsection{Partial-Observation Protocol}
\label{app:data:partial}
Across all benchmarks, downstream robustness is evaluated at observed-information levels $c \in \{0.05, 0.2, 0.4, 0.6, 0.8\}$ under benchmark-specific partial-observation processes. The source of these evaluation-time partial-observation patterns differs by benchmark family: right-censored time-series and image benchmarks use predefined schemes, while non-right-censored time-series and static tabular benchmarks use structured processes. Across all benchmarks, partial observation is coupled to task-relevant structure – making the refined view systematically more informative for the downstream task – precisely the regime where martingale consistency is most beneficial (Proposition~\ref{prop:excess_risk}).

\paragraph{Time-series benchmark (right-censored).}
For this benchmark, the partially observed view is obtained by right-censoring the time axis after an observed prefix. Writing $L \in \{0,\ldots,T\}$ for the observed prefix length, the mask is
\begin{equation}
M_{t,d}=\mathbf{1}\{t \le L\}, \qquad t=1,\ldots,T,\ d=1,\ldots,D.
\end{equation}
The complete input acts as the refined target view. This construction directly matches the interpretation of information refinement through increasing history length.
In \emph{T-SIM-RC}, label information is concentrated near sequence end by construction, so right-censoring preferentially removes the most predictive timesteps and makes cross-view coherence particularly consequential for downstream robustness.

\paragraph{Time-series benchmark (non-right-censored).}
For the non-right-censored time-series benchmark, missingness is generated by a structured temporal process. For example $i$ and time index $t$, we use a low-rank logistic model
\begin{equation}
\Pr(M_{i,t}=1)=\sigma\!\bigl(b_t + s_i + z_i^\top \ell_t - \beta q_t + \delta(c)\bigr),
\end{equation}
where $q_t$ is a temporal-importance profile  (for T-SIM, derived from the synthetic construction as a binary indicator over label-relevant timesteps; for real time-series datasets, estimated on the SSL training split), $z_i \sim \mathcal{N}(0,I_r)$ is an example-level latent factor, $s_i \sim \mathcal{N}(0,\tau^2)$ is a severity term, and $\delta(c)$ is a scalar shift chosen so that the expected observed fraction matches target completeness $c$. The minus term $-\beta q_t$ explicitly couples missingness to temporal importance, so timesteps that are more predictive are also more likely to be missing under the target process. In the reported setup we use rank $r=3$, base-logit scale $0.04$, correlation-loading scale $0.05$, importance strength $\beta=5.0$, and severity standard deviation $\tau=0.12$.

\paragraph{Static tabular benchmark.}
The static tabular benchmark uses the same parameterization, with feature index $j$ replacing time index $t$:
\begin{equation}
\Pr(M_{i,j}=1)=\sigma\!\bigl(b_j + s_i + z_i^\top \ell_j - \beta q_j + \delta(c)\bigr).
\end{equation}
Here $q_j$ is a feature-importance profile (for \emph{S-SIM}, derived from the synthetic construction; for real static tabular datasets, estimated on the training split), and $z_i,s_i,\delta(c)$ play the same roles as above. The reported settings again use rank $3$, base scale $0.04$, correlation scale $0.05$, importance strength $5.0$, and severity standard deviation $0.12$.

\paragraph{Image benchmarks.}
For the image benchmarks, partial observation is induced by patch masking on a regular patch grid. Let $P_{i,p}\in\{0,1\}$ denote the patch-level observation mask for image $i$ and patch index $p$, with
\begin{equation}
\frac{1}{n_{\mathrm{patch}}}\sum_{p=1}^{n_{\mathrm{patch}}} P_{i,p} \approx c,
\end{equation}
where $c$ is the target observed fraction. In the reported CIFAR-10 and STL-10 experiments we use a center-biased masking scheme, implemented by sampling dropped square patch blocks with probabilities weighted toward the image center according to
\begin{equation}
w_p \propto \exp\!\left(-\gamma\,\widetilde d_p^{\,2}\right),
\qquad \gamma=3.0,
\end{equation}
where $\widetilde d_p^{\,2}$ denotes the squared distance of patch $p$ from the image center after normalization by the maximum squared distance on the patch grid. The image-level mask is obtained by expanding the patch mask over pixels, and the refined target view is always the complete image.

\section{Model Implementation Details}
\label{app:impl}

\subsection{Model architectures and objective families}
Across all reported experiments, training starts from a complete input $X$ and samples a mask $M$ to construct a coarse view. For temporal and tabular benchmarks, $X$ is an array in $\mathbb{R}^{T\times D}$, with static tabular data treated as the special case $T=1$, and the coarse input is $X\odot M$ with $M\in\{0,1\}^{T\times D}$. For image benchmarks, $X\in\mathbb{R}^{C\times H\times W}$ and $M$ is a sampled patch mask applied directly to the image. In all cases, the refined view used by the martingale branch is the complete input.
The base implementation consists of an encoder $f_\theta$, a prediction head $g_\phi$, and optionally an imputation module $q_\psi$. For temporal and tabular data, $f_\theta$ is a mask-aware encoder applied to masked sequence- or feature-inputs. For images, $f_\theta$ is a visual encoder applied to masked images; in the reported CIFAR-10 and STL-10 experiments, the semi-self-supervised and SimCLR models use a ResNet-18 \cite{He2016_ResNet} backbone. We write the resulting coarse and refined representations generically as
\begin{equation}
z_{\mathcal{F}_1}:=f_\theta(X\odot M),
\qquad
z_{\mathcal{F}_2}:=f_\theta(X),
\end{equation}
with temporal and tabular models implemented using a mask-aware encoder, while for images the mask is applied to the pixels before the encoder is called.

\paragraph{Base framework.}
The core benchmark families use two objective types. In the reconstruction-based setting, the prediction head predicts the complete input from the encoder representation, and the reconstruction term is evaluated on the complete input:
\begin{equation}
\mathcal{L}_{\mathrm{pred}}^{\mathrm{rec}}
:=
\frac{1}{|\Omega|}\left\|g_\phi^{\mathrm{rec}}\bigl(z_{\mathcal{F}_2}\bigr)-X\right\|_2^2,
\end{equation}
where $|\Omega|$ denotes the number of input coordinates or pixels. In the semi-self-supervised setting, the prediction head maps the representation to class logits and the classification term is cross-entropy on the complete input:
\begin{equation}
\mathcal{L}_{\mathrm{pred}}^{\mathrm{cls}}
:=
- \log \frac{\exp\bigl(g_\phi^{\mathrm{cls}}(z_{\mathcal{F}_2})_Y\bigr)}{\sum_{k=1}^K \exp\bigl(g_\phi^{\mathrm{cls}}(z_{\mathcal{F}_2})_k\bigr)}.
\end{equation}
Thus, in the base framework these prediction terms are evaluated on the complete example, while the masked view enters through the imputation and martingale-consistency terms. In the reported time-series and static tabular experiments, the semi-self-supervised models use supervised classification and the fully self-supervised base models use the reconstruction objective. 
For time-series and static tabular classification, the prediction head acts on masked-mean pooled encoder representations; for image classification, it acts on the direct image encoder output. For reconstruction, the prediction head acts on unpooled encoder outputs so that the complete input can be predicted coordinate-wise or pixel-wise.

\subsection{Imputation module}
Beyond these prediction or representation-learning terms, martingale-based models may use an imputation module $q_\psi$ that predicts the missing part of the input from the coarse-view representation. For temporal and tabular data, this means missing coordinates in the original sequence/feature domain; for images, it means missing pixels in the masked image. Writing $X_{\mathrm{imp}}=q_\psi(z_{\mathcal{F}_1})$, the imputation loss is
\begin{equation}
\mathcal{L}_{\mathrm{imp}}(\theta,\psi;X,M)
 :=
\frac{\displaystyle \sum_{j}(1-M_j)\bigl((X_{\mathrm{imp}})_{j}-X_{j}\bigr)^2}
{\displaystyle \sum_{j}(1-M_j)+\varepsilon},
\end{equation}
where the index $j$ ranges over coordinates or pixels and $\varepsilon>0$ prevents division by zero. Observed entries are copied directly from the input, so the imputer is trained only on genuinely missing components. In the reported experiments this module serves two purposes: it contributes an explicit imputation term to the loss, and it generates stochastic completions for the martingale regularizer.

\subsection{Masking schemes during SSL model training}
\label{app:impl:masking}
Across all reported experiments, training is nested: the model receives complete inputs, but a coarse mask is deliberately sampled to construct the partial-information branch, while the refined branch remains the complete input. The coarse-view completeness is randomized batch-by-batch over $[0.05,1.0]$, covering both severe and mild masking regimes.

For structured temporal and static tabular benchmarks, training masks are not drawn directly from the evaluation-time process. Instead, we estimate feature- or timestep-importance scores on the SSL training split, generate calibration masks on the held-out prior-fitting partition at completeness $0.5$, fit a mask prior from those observed masks, and sample coarse-view masks from that prior during training. This approximates deployment-time missingness without reusing evaluation examples. For the right-censored temporal and image benchmarks, masks are drawn directly from the predefined masking family: right-censoring prefixes for \emph{T-SIM-RC} and center-biased patch masks for the image setting.

\subsection{Martingale-consistent objectives in the reported experiments}
\label{app:impl:mart}

In all reported experiments, $x_{\mathcal{F}_2}=x$ is the full observation and $x_{\mathcal{F}_1}=x \odot M$ is the coarse view. The two conditionally independent refinements required by Eq.~\eqref{eq:2mc} are produced by stochastic imputer-based completion, merging imputed missing entries with the observed coordinates. For time-series and tabular experiments, writing $r\in\{a,b\}$,
\begin{equation}
\tilde{x}_r := x_{\mathcal{F}_1}+\xi_r\odot(1-M),
\qquad
\tilde{z}_r := f_\theta(\tilde{x}_r),
\qquad
\hat{x}_r := x \odot M \;+\; q_\psi(\tilde{z}_r)\odot(1-M),
\end{equation}
where $\xi_a$ and $\xi_b$ are independent Gaussian noise draws. In image experiments, the imputer is fed a noisy encoder representation of $x_{\mathcal{F}_1}$ instead.

\paragraph{Prediction-space variant.}
Writing $z = f_\theta(x_{\mathcal{F}_1})$ and $z_a = f_\theta(\hat{x}_a)$, $z_b = f_\theta(\hat{x}_b)$ for the representations passed to the prediction head, coarse and refined predictions $u = g_\phi(z)$, $v_a = g_\phi(z_a)$, $v_b = g_\phi(z_b)$ are substituted into Eq.~\eqref{eq:2mc}, giving $\widehat{\mathcal{L}}_{\mathrm{mart,pred}}^{2\mathrm{MC}}$.

\paragraph{Latent-space variant.}
Latent-space variants skip the prediction head in the martingale term and apply Eq.~\eqref{eq:2mc} directly to the same coarse and refined representations, giving $\widehat{\mathcal{L}}_{\mathrm{mart,lat}}^{2\mathrm{MC}} := (z-z_a)^\top(z-z_b)$. The SimCLR and BYOL variants use only this latent-space construction.

\subsection{Training schedule and optimization}
The overall training loss has the form
\begin{equation}
\mathcal{L}_{\mathrm{total}}
=
\mathcal{L}_{\mathrm{pred}}
+
\lambda_{\mathrm{imp}}\,\mathcal{L}_{\mathrm{imp}}
+
\lambda_{\mathrm{mart}}^{\mathrm{eff}}\,\mathcal{L}_{\mathrm{mart}},
\end{equation}
where $\mathcal{L}_{\mathrm{mart}}$ is either the prediction-space or latent-space martingale penalty depending on the model variant. The effective martingale weight follows a warmup--ramp schedule:
\begin{equation}
\lambda_{\mathrm{mart}}^{\mathrm{eff}}(s)
=
\lambda_{\mathrm{mart}} \cdot \gamma(s),
\end{equation}
where $\gamma(s)=0$ during the warmup phase, then increases linearly to $1$ over the ramp interval, and remains equal to $1$ thereafter. This schedule is used to reduce early optimization instability.

\subsection{Training and Model Hyperparameters}
\label{app:impl:hyperparams}

\paragraph{Time-series and static tabular experiments.}
We select $\lambda_{\mathrm{imp}}$ and $\lambda_{\mathrm{mart}}$ using a grid search over $\lambda_{\mathrm{imp}}\in\{10^{-4},10^{-3},\ldots,10^{3}\}$ and $\lambda_{\mathrm{mart}}\in\{10^{-4},10^{-3},\ldots,10^{4}\}$. For each pair $(\lambda_{\mathrm{imp}},\lambda_{\mathrm{mart}})$, we train the martingale-consistent variant of each SSL model on the full training split; to score this pair, we fit a linear probe on $80\%$ of the training data and validate on the remaining $20\%$, and we keep the $(\lambda_{\mathrm{imp}},\lambda_{\mathrm{mart}})$ with the largest mean validation accuracy, averaged over completeness levels $c\in\{0.8,0.6,0.4,0.2,0.05\}$. The Base + imputation model reuses the selected $\lambda_{\mathrm{imp}}$ from the matched martingale model family. The remaining settings are shared across reported models and datasets. We use learning rate $10^{-3}$ and weight decay $10^{-4}$; training uses $500$ steps per run and batch size $256$. The encoder is a single MLP acting on the full feature vector at each time index. Its layer widths are $128$--$64$--$32$, with GELU activation and dropout $0.1$; representations are $32$-dimensional masked means over observed positions. In the semi-self-supervised setting, a single linear map reads off class logits from the pooled $32$-dimensional representation. In the fully self-supervised (reconstruction) and imputation objectives, we use linear readouts (shared across times), each a single map from hidden features to the input feature dimension. For martingale regularization, the training coarse view samples an observed-information rate in $[0.05,1.0]$ each batch. The martingale term uses $100$ warmup steps and a $400$-step linear ramp; EMA-based variants use decay $0.97$ (full-target EMA for latent martingale, head-only EMA for prediction-space martingale). Two stochastic refinements use Gaussian completion noise with scale $0.25$. In the SimCLR and BYOL tabular runs, the contrastive and bootstrap objectives are computed from two independently sampled partial views obtained through the same (fitted) masking mechanism as for martingale regularization. The projection MLP is $32\!\to\!128\!\to\!64$ with temperature $0.1$ for SimCLR, and BYOL adds a predictor MLP $64\!\to\!256\!\to\!64$ with target decay $0.99$. Downstream probes use a linear head, $10$ probe epochs, learning rate $10^{-2}$, weight decay $0$, and batch size $1000$; the probe is trained on full-observation embeddings and tested at $c\in\{0.8,0.6,0.4,0.2,0.05\}$. Both SSL pretraining and linear probe training use the AdamW optimizer \cite{LoshchilovHutter2019_AdamW}.

\paragraph{Image experiments.}
For image experiments, we implement the encoder as a ResNet-18 \cite{He2016_ResNet} trained from scratch: we use the standard convolutional stack through global average pooling and omit the final linear layer that would map the pooled feature vector to class logits on a $1000$-class ImageNet head, so the representation is the $512$-dimensional vector produced after pooling. CIFAR-10 is used at its native $32\!\times\!32$ resolution, and STL-10 is resized from $96\!\times\!96$ to $32\!\times\!32$. SSL pretraining uses batch size $256$ and runs for $50$ epochs with learning rate $10^{-3}$ and weight decay $10^{-4}$. Masks use $4\!\times\!4$ patches, center bias $3.0$, and we report at completeness $c\in\{0.8,0.6,0.4,0.2,0.05\}$. In the semi-self-supervised family, the Base model uses a linear classifier on encoder features; Base + imputation adds an MLP imputer with $\lambda_{\mathrm{imp}}=1$; and the latent-martingale variant uses the same imputer with $\lambda_{\mathrm{imp}}=1$ and $\lambda_{\mathrm{mart}}=0.01$, a $5$-epoch martingale warmup, a $12$-epoch linear ramp, and martingale clip $100$. The imputer maps $512$-dimensional latents to full $3\!\times\!32\!\times\!32$ pixels and uses refinement noise scale $0.05$. SimCLR uses standard full-image augmentations (random resized crop in scale $[0.2,1.0]$, horizontal flip with probability $0.5$, grayscale with probability $0.2$), a projection head $512\!\to\!512\!\to\!128$, and temperature $0.2$ \cite{Oord2018_CPC,Chen2020_SimCLR}; SimCLR + imputation uses $\lambda_{\mathrm{imp}}=1$, and SimCLR + imputation + latent martingale uses $\lambda_{\mathrm{imp}}=1$ and $\lambda_{\mathrm{mart}}=1$ with the same warmup, ramp, and martingale clip settings. The downstream probe is trained on unmasked features with a linear head for $25$ epochs, learning rate $10^{-2}$, weight decay $10^{-4}$, and batch size $1024$.

All experiments were executed on a shared cluster with one NVIDIA A100 GPU per task (40--80\,GB VRAM), with host memory ranging between 20--90\,GB. In our measurements, runtime of individual training runs is on the order of a few seconds for the static tabular experiments and around one hour for the image experiments.

\end{document}